\documentclass{article}

\usepackage{arxiv}
\usepackage{natbib}

\usepackage{amssymb}
\usepackage{amsthm}
\usepackage{float}

\usepackage{lineno,hyperref}
\usepackage[utf8]{inputenc}
\usepackage{graphicx}
\usepackage{multicol}
\usepackage[most]{tcolorbox}
\usepackage[utf8]{inputenc}
\usepackage[english]{babel}
\usepackage{nccmath}
\usepackage{empheq} 
\usepackage{algorithm}
\usepackage{algpseudocode}
\usepackage{stmaryrd}
\algnewcommand{\Inputs}[1]{%
  \State \textbf{Inputs:}
  \Statex \hspace*{\algorithmicindent}\parbox[t]{.8\linewidth}{\raggedright #1}
}
\algnewcommand{\Initialize}[1]{%
  \State \textbf{Initialize}
  \Statex \hspace*{\algorithmicindent}\parbox[t]{.8\linewidth}{\raggedright #1}
}

\usepackage{amsthm}
\usepackage{caption}
\usepackage{array} 
\theoremstyle{definition}
\newtheorem{definition}{Definition}
\newtheorem{theorem}{Theorem}

\newtheorem{lemma}{Lemma}
\newtheorem{assumption}{Assumption}

\usepackage{subcaption}
\usepackage{placeins}

\def\nbR{\ensuremath{\mathrm{I\! R}}}
\usepackage{ulem}
\usepackage[siunitx, RPvoltages]{circuitikz}
\usetikzlibrary{arrows.meta, positioning}
\title{Distributionally Robust Geometric Joint Chance-Constrained Optimization: Neurodynamic Approaches}

\author{%
Ange Valli$^{1,}\thanks{corresponding author}$ \quad Siham Tassouli$^{2}$ \quad Abdel Lisser$^{1,3}$\\
$^1$Université Paris-Saclay, CNRS, CentraleSupélec, Laboratoire des signaux et systèmes\\
$^2$ENAC, OPTIM\\
$^3$Université Paris-Saclay, CNRS, CentraleSupélec, Fédération de Mathématiques de CentraleSupélec\\
\texttt{\{ange.valli,abdel.lisser\}@l2s.centralesupelec.fr}\\
\texttt{siham.tassouli@enac.fr}
}

\begin{document}

\maketitle

\begin{abstract}
This paper proposes a two-time scale neurodynamic duplex approach to solve
distributionally robust geometric joint chance-constrained optimization
problems. The probability distributions of the row vectors are not known in
advance and belong to a certain distributional uncertainty set. In our paper, we
study three uncertainty sets for the unknown distributions. The neurodynamic
duplex is designed based on three projection equations. The main contribution of
our work is to propose a neural network-based method to solve distributionally
robust joint chance-constrained optimization problems that converges in
probability to the global optimum without the use of standard state-of-the-art
solving methods. We show that neural networks can be used to solve multiple
instances of a problem. In the numerical experiments, we apply the proposed
approach to solve a problem of shape optimisation and a telecommunication
problem.
\end{abstract}


{\def\thefootnote{}\footnotetext{Websites: \textbf{Ange Valli:} \url{https://angevalli.github.io/}, \textbf{Siham Tassouli:} \url{https://optim.recherche.enac.fr/author/siham-tassouli/}, \textbf{Abdel Lisser:} \url{https://l2s.centralesupelec.fr/u/lisser-abdel/}}}

\section{Introduction}
\noindent Chance-constrained programming appears with the increased need to include uncertainty in complex decision-making models. It was introduced for the first time by Charnes \& Cooper \cite{charnes1959chance}. Since then, chance-constrained optimization has been widely studied, and the range of applications is very large, e.g., resource allocation in communication and network systems \cite{kandukuri2002optimal,Hadi:17}, information theory \cite{posada2020gait}, chemical engineering \cite{SONMEZ1999297}, computational finance \cite{kabanov2004geometric, ZHANG2012217}, metal cutting optimization \cite{Dupacov2010StochasticGP}, spatial gate sizing \cite{singh2008geometric}, profit maximization  \cite{kojic2018solving} and biochemical systems \cite{xu2014improved}, portfolio selection \cite{HAN20171189}, energy systems operations \cite{HUO2021107153}, water quality management \cite{dhar2006chance} and transportation problems \cite{sluijk2023chance}. In this paper, we study chance-constrained geometric programs. Liu et al. \cite{LIU2016687} propose some convex based approximations to come up with lower and upper bounds for geometric programs with joint probabilistic constraints when the stochastic parameters are normally distributed and pairwise independent. Shiraz et al. \cite{khanjani2017fuzzy} use a duality algorithm to solve fuzzy chance-constrained geometric programs. Tassouli \& Lisser \cite{TASSOULI2023765} propose a neurodynamic approach to solve geometric programs with joint probabilistic constraints with normally distributed coefficients and independent matrix row vectors. Liu et al. \cite{liu2022distributionally} propose convex approximation based algorithms to solve distributionally robust geometric programs with individual and joint chance constraints.

Geometric programming is a method for solving a class of nonlinear problems. It
is used to minimize functions that are in the form of posynomials subject to
constraints of the same type. It was introduced for the first time by Duffin et
al. \cite{Duffin67}. Since then geometric programming was employed to solve
several optimization problems, e.g, resource allocation in communication and
network systems \cite{kandukuri2002optimal,Hadi:17}, information theory
\cite{posada2020gait, muqattash2006performance}, chemical engineering
\cite{SONMEZ1999297}, computational finance \cite{kabanov2004geometric,
ZHANG2012217}, metal cutting optimization \cite{Dupacov2010StochasticGP},
spatial gate sizing \cite{singh2008geometric}, profit maximization
\cite{kojic2018solving} and biochemical systems \cite{xu2014improved}.

In this paper, we are interested in solving joint chance-constrained geometric
optimization problems. We study the case where the distribution of the random
parameters is unknown, aka distributionally robust optimization. In fact, we may
only know partial information about the statistical properties of the stochastic
parameters. El Ghaoui \& Lebret \cite{el1997robust} use second-order cone
programming to solve least-squares problems where the stochastic parameters are
not known but bounded. Bertsimas \& Sim \cite{bertsimas2004price} introduce a
less conservative approach to solve linear optimization problems with uncertain
data. Bertsimas \& Brown \cite{bertsimas2009constructing} propose a general
scheme for designing uncertainty sets for robust optimization. Wiesemann et al.
\cite{wiesemann2014distributionally} propose standardized ambiguity sets for
modeling and solving distributionally robust optimization problems. Peng et al.
\cite{peng2021games} study one density-based uncertainty set and four
two-moments based uncertainty sets to solve games with distributionally robust
joint chance constraints. Cheng et al. \cite{cheng2014distributionally} solve a
distributionally robust quadratic knapsack problem. Dou \& Anitescu
\cite{DOU2019294} propose a new ambiguity set tailored to unimodal and seemingly
symmetric distributions to deal with distributionally robust chance constraints.
Li \& Ke \cite{LI2019452} approximate a distributionally robust chance
constraint by the worst-case Conditional Value-at-Risk. Hanasusanto et al.
\cite{HANASUSANTO20166} approximate two-stage distributionally robust programs
with binary recourse decisions. Georghiou et al. \cite{georghiou2020primal}
propose a primal-dual lifting scheme for the solution of two-stage robust
optimization problems. Xia et al. \cite{XIA2023315} propose an efficient
solution approach to the Nash equilibrium of the distributionally robust
chance-constrained games. Yang \& Li \cite{yang2023distributionally} use
Sinkhorn ambiguity set to solve distributionally robust chance-constrained
optimization problems.

Recent papers have considered the use of distributionally robust approaches in
transportation network optimization problems \cite{dai2020distributionally},
multistage distribution system planning \cite{zare2018distributionally},
portfolio optimization problems \cite{fonseca2012robust,
wang2023distributionally}, planning and scheduling \cite{SHANG201853}, risk
measures \cite{postek2016computationally}, multimodal demand problems
\cite{hanasusanto2015distributionally}, design of hub location
\cite{zhao2023distributionally}, studying flexibility of networks
\cite{rayati2022distributionally}, appointment scheduling \cite{ZHANG2017139},
vehicle routine problems \cite{ghosal2020distributionally} and energy and
reserve dispatch \cite{ORDOUDIS2021291}.

The use of neural networks to solve optimization problems has been actively
studied since the 1980s when the idea was first introduced by Tank \& Hopfield
\cite{tank2003simple}. Xia \& Wang \cite{xia2004recurrent} present a recurrent
neural network for solving nonlinear convex programming problems subject to
nonlinear inequality constraints. Wang \cite{WANG1994629} proposes a
deterministic annealing neural network for convex programming. Nazemi \& Omedi
\cite{NAZEMI20131} presents a neural network model for solving the shortest path
problems. Tassouli \& Lisser \cite{TASSOULI2023765} propose a recurrent neural
network to solve geometric joint chance-constrained optimization problems. Wu \&
Lisser \cite{wu2023deep} proposes a physics-informed neural network to solve
multiple linear programs formulated as initial value problems. Flexible
architectures of neural networks have been developed to overcome the limitations
of standard feed-forward neural networks and vanilla recurrent neural networks:
physics-informed neural networks (PINNs) \cite{raissi2019physics}, conditional
recurrent neural networks (RNNs), long short-term memories (LSTMs), gated
recurrent units (GRUs) \cite{lauria2024conditional}. Conditional recurrent
neural networks were developed for computer vision applications
\cite{karpathy2015deep,vinyals2015show}. Recently, they found applications in
neuroscience \cite{driscoll2024flexible}, optics \cite{lauria2024conditional}
and big data \cite{fan2025conditional}, allowing to incorporate external
features independent in time, which condition the sequential data learned by the
model.

 In addition to the significant accomplishments achieved by individual recurrent
 neural networks, it is important to note that one-time-scale RNNs have
 limitations when it comes to constrained global optimization problems and more
 general problem domains. The dynamic behaviors of one-time-scale RNNs can
 exhibit drastic changes and become unpredictable when dealing with certain
 optimization problems. Neuroscience studies have provided evidence of the
 existence of two-time-scales in the brain, where different processes operate at
 different time scales \cite{spitmaan2020multiple}. Therefore, a neurodynamic
 model with two-time-scales is considered more biologically plausible for
 emulating brain functions than a model with only one-time-scale. Xia et al.
 \cite{xia2023two} prove that two-timescale RNNs converge faster than the
 one-timescale RNNs. This paper proposes a two-timescale duplex neurodynamic
 approach for distributionally robust joint chance-constrained optimization
 problems, which is formulated using a biconvex reformulation. Unlike other
 existing methods that give lower or upper bounds to this kind of problems, the
 proposed approach employs two recurrent neural networks that operate
 collaboratively at two different timescales and converge almost surely to a
 global optimal solution of the given distributionally robust optimization
 problem. 

 In this paper, we rely on the deterministic reformulations proposed in
 \cite{liu2022distributionally} and propose two-timescale neural network-based
 solutions without using any convex approximations or relaxations. We also rely
 on the one-time scale neural network proposed in \cite{TASSOULI2023765} to
 create our two-time scale neural network.

The main contributions of our work are threefold.
\begin{itemize}
    \item[(i)] On the formulation side, we reformulate two neurodynamic approaches to solve the equivalent deterministic problems of the initial robust programs.
    \item[(ii)] On the theoretical side, we show that 
    the proposed neurodynamic approaches are stable and convergent.
    \item[(iii)] On the numerical side, we show that the proposed neurodynamic methods cover well the risk area induced by the distributionally robust chance constraints.
\end{itemize}
The rest of the paper is organized as follows. In Section
\ref{sec:problem_statement}, we study two uncertainty sets to solve a
distributionally robust geometric chance-constrained optimization problem and
give four deterministic equivalent problems. In Section
\ref{sec:single_recurrent_neural_network}, we propose a recurrent neural network
to solve the first three resulting problems and prove its convergence and
stability. We consider in Section \ref{sec:two-time_scale_neurodynamic_duplex} a
duplex of two two-timescale recurrent neural networks to solve the last
deterministic problem and prove its convergence almost surely to the global
optimum. In Section \ref{sec:numerical_experiments}, we evaluate the performances of the proposed neurodynamic
approaches by solving a shape optimization problem and a telecommunication
problem.

\section{Problem statement and reformulation}
\label{sec:problem_statement}

\noindent The mathematical reformulations of the robust problems considered in this section are inspired from \cite{liu2022distributionally}, however, our theoretical results and algorithms are completely different.
\\
A general form of a geometric program is given as follows 
\begin{eqnarray}
&\min\limits_{t \in \mathbb{R}_{++}^M} &\sum\limits_{i=1}^{I_0} c_i^0 \prod_{j=1}^M t_j^{a_{ij}^0},  \label{P1}\\
&{\text{s.t}}& 
      \sum\limits_{i=1}^{I_k} c_i^k \prod_{j=1}^M t_j^{a_{ij}^k} \leq 1, k=1, ...., K, \notag
\end{eqnarray}
where  $c_i^k$, $i=1,...,I_k$, $k=1, ...., K$ are positive constants and the exponents $a_{ij}^k, i=1,...,I_k, j=1,...,M, k=1, ..., K$ are real constants.
\\
In this paper, we consider the case where the coefficients $c_i$ are not known. Consequently, we reformulate the optimization problem (\ref{P1}) as follows 
\begin{eqnarray}
&\min\limits_{t \in \mathbb{R}_{++}^M} &\sup\limits_{\mathcal{F}_0 \in \mathcal{D}_0 } \mathbb{E}_{\mathcal{F}_0}\left [\sum\limits_{i=1}^{I_0} c_i^0 \prod_{j=1}^M t_j^{a_{ij}^0}\right ],  \tag{$JCP$} \label{JCP}\\
&{\text{s.t}}& 
          \inf\limits_{\mathcal{F} \in \mathcal{D} }\mathbb{P}_{\mathcal{F}}\left (\sum\limits_{i=1}^{I_k} c_i^k \prod_{j=1}^M t_j^{a_{ij}^k} \leq 1, k=1, ...., K \right) \geq 1-\epsilon,  \notag
\end{eqnarray}
where $\mathcal{F}_0$ is the probabilistic distribution of vector $C^0 = (c_1^0, .., c_{I_0}^0)^T$,  $\mathcal{F}$ is the joint distribution for $C^1 = (c_1^1, .., c_{I_1}^1)^T$, ...,  $C^k = (c_1^k, .., c_{I_k}^k)^T$, $\mathcal{D}_0$ is the uncertainty set for the probability distribution $\mathcal{F}_0$, $\mathcal{D}$ is the uncertainty set for the probability distribution $\mathcal{F}$ and $1-\epsilon$, $\epsilon \in (0, 0.5] $, is the confidence parameter for the joint constraint.
\\
This paper considers the distributionally robust geometric programs (\ref{JCP})
using two different sets of uncertainty. The first set focuses on uncertainties
in distributions, considering both known and unknown first two order moments.
The second set incorporates first-order moments along with nonnegative support
constraints.

\subsection{Uncertainty Sets with First Two Order Moments}
\noindent We first consider that the mean vector of $C^k$, $k=1,...,K$ lies in
an ellipsoid of size $\gamma_{1}^k \geq 0$ with center $\mu_k$ and that the
covariance matrix of  $C^k$, $k=1,...,K$ lies in a positive semidefinite cone of
center  {$\Sigma_k = \left\{ \sigma_{i,j}^k, i = 1, ..., I_k, i=1, ..., M
\right\}$}. We define for every $k=1,...,K$, $\mathcal{D}^2_k(\mu_k, \Sigma_k)=
\left\{\mathcal{F}_k \left|
\begin{array}{rcl} 
  (\mathbb{E}_{\mathcal{F}_k}[{C^k}]-\mu_k)^T \Sigma_k^{-1}  (\mathbb{E}_{\mathcal{F}_k}[{C^k}]-\mu_k)\leq \gamma_1^k\\
  \text{COV}_{\mathcal{F}_k}({C^k}) \preceq \gamma_2^k\Sigma_k\
\end{array}\right.
\right\}$, where $\mathcal{F}_k$ is the probability distribution of $C^k$, $\gamma_2^k \geq 0$ and $\text{COV}_{\mathcal{F}_k}$ is a covariance operator under probability distribution ${\mathcal{F}_k}$ of $C^k$.\\
\\
Based on whether the row vectors $C^k$, $k=1,..,K$ are mutually independent or dependent, we have two cases.

\subsubsection{Case \texorpdfstring{(\ref{JCP})}\ \ with Jointly Independent Row Vectors}
\begin{assumption}
\label{assmp1}
We assume that $\mathcal{D} = \left\{ \mathcal{F} |  \mathcal{F} = \mathcal{F}_1 \mathcal{F}_2 ...  \mathcal{F}_K \right\}$, where $\mathcal{F}$ is the joint distribution for mutually independent random vectors $C^1$, $C^2$, ..., $C^K$ with marginals $\mathcal{F}_1$, $\mathcal{F}_2 $, ..., $\mathcal{F}_K$. 
\end{assumption}
 \begin{theorem}
Given Assumption \ref{assmp1}, (\ref{JCP}) is equivalent to
\begin{align} 
& \min \limits_{t \in \mathbb{R}_{++}^M, y \in \mathbb{R}_{+}^K } &&  \sum_{i =1}^{I_0} \mu_i^0 \prod_{j=1}^M t_j^{a_{ij}^0} + \sqrt{\gamma_1^0} \sqrt{\sum\limits_{i =1}^{ I_0}\sum\limits_{l =1}^{I_0} \sigma_{i,l}^0 \prod_{j=1}^M t_j^{a_{ij}^0+a_{lj}^0}} \tag{$JCP_{ind}$} \label{JCP_ind}\\
&\text{s.t} &&\sum_{i =1}^{I_k} \mu_i^k \prod_{j=1}^M t_j^{a_{ij}^k}+\sqrt{\gamma_1^k} \sqrt{\sum\limits_{i =1}^{ I_k}\sum\limits_{l =1}^{I_k} \sigma_{i,l}^k \prod_{j=1}^M t_j^{a_{ij}^k+a_{lj}^k}} \notag\\
& &&+ \sqrt{\frac{y_k}{1-y_k}}\sqrt{\gamma_2^k} \sqrt{\sum\limits_{i =1}^{ I_k}\sum\limits_{l =1}^{I_k} \sigma_{i,l}^k \prod_{j=1}^M t_j^{a_{ij}^k+a_{lj}^k}} \leq 1 \quad \forall k=1,..., K \notag\\
& &&\prod_{k=1}^K y_k \geq 1-\epsilon, \quad 0 < y_k \leq 1 \quad \forall k=1,...,K \notag
\end{align}   
\end{theorem}
\begin{proof}
As the row vectors $C^k$, $k=1,..., K$ are mutually independent, (\ref{JCP}) is written equivalently by introducing $K$ nonegative auxiliary variables $y_k$ as \cite{TASSOULI2023765}.
\begin{eqnarray}
&\min\limits_{t \in \mathbb{R}_{++}^M}   &  \sup\limits_{\mathcal{F}_0 \in \mathcal{D}_0 } \mathbb{E}_{\mathcal{F}_0}\left [\sum\limits_{i=1}^{I_0} c_i^0 \prod_{j=1}^M t_j^{a_{ij}^0}\right ],  \notag\\
&{\text{s.t}}& 
          \inf\limits_{\mathcal{F}_k \in \mathcal{D}_k }\mathbb{P}_{\mathcal{F}_k}\left (\sum\limits_{i=1}^{I_k} c_i^k \prod_{j=1}^M t_j^{a_{ij}^k} \leq 1 \right) \geq y_k, \forall k=1, ...., K,  \notag \\
& & \prod_{k=1}^K y_k \geq 1-\epsilon, \text{ }  0 < y_k \leq 1 \text{ , } k=1,...,K. \notag
\end{eqnarray}
\\
By Theorem 1 in \cite{liu2022distributionally}, we conclude that (\ref{JCP}) is equivalent to (\ref{JCP_ind}).
\end{proof}
\noindent Problem (\ref{JCP_ind}) is not convex. By applying the logarithmic transformation $r_j = log(t_j)$, $j=1,...,M$ and $x_k = \text{log}(y_k)$, $k=1, ..., K$, we have the following equivalent reformulation of  
(\ref{JCP_ind})
\begin{eqnarray} 
&\min \limits_{r \in \mathbb{R}^M, x \in \mathbb{R}^K } &  \sum_{i =1}^{I_0} \mu_i^0 \text{exp}\left\{\sum_{j=1}^M {a_{ij}^0}r_j\right\} + \sqrt{\gamma_1^0} \sqrt{\sum\limits_{i =1}^{ I_0}\sum\limits_{l =1}^{I_0} \sigma_{i,l}^0  \text{exp} \left\{\sum_{j=1}^M (a_{ij}^0+a_{lj}^0)r_j\right\}} \tag{$JCP_{ind}^{log}$} \label{JCP_ind^log} \\
&\text{s.t}& 
        \sum_{i =1}^{I_k} \mu_i^k  \text{exp}\left\{\sum_{j=1}^M {a_{ij}^k}r_j\right\}+\sqrt{\gamma_1^k}  \sqrt{\sum\limits_{i =1}^{ I_k}\sum\limits_{l =1}^{I_k} \sigma_{i,l}^k  \text{exp} \left\{\sum_{j=1}^M (a_{ij}^k+a_{lj}^k)r_j\right\}} \notag\\
        & & + \sqrt{\gamma_2^k} \sqrt{\sum\limits_{i =1}^{ I_k}\sum\limits_{l =1}^{I_k} \sigma_{i,l}^k  \text{exp} \left\{\sum_{j=1}^M (a_{ij}^k+a_{lj}^k)r_j+\text{log}\left(\frac{e^{x_k}}{1-e^{x_k}} \right)\right\}} \leq 1 \quad \forall k=1,..., K \notag\\
&  &\sum_{k=1}^K x_k \geq \text{log}(1-\epsilon), \quad x_k \leq 0 \quad \forall k=1,...,K \notag
\end{eqnarray}   

\begin{theorem} \cite{liu2022distributionally}
   If $\sigma_{i,l}^k \geq 0$ for all $i$, $l$ and $k$, problem (\ref{JCP_ind^log}) is a convex programming problem.
\end{theorem}

\subsubsection{Case \texorpdfstring{(\ref{JCP})}\ \ with Jointly Dependent Row Vectors.}
\noindent In this case, (\ref{JCP}) is equivalent to  \cite{liu2022distributionally}

\begin{eqnarray} 
&\min \limits_{t \in \mathbb{R}_{++}^M, y \in \mathbb{R}_{+}^K } &  \sum_{i =1}^{I_0} \mu_i^0 \prod_{j=1}^M t_j^{a_{ij}^0} + \sqrt{\gamma_1^0} \sqrt{\sum\limits_{i =1}^{ I_0}\sum\limits_{l =1}^{I_0} \sigma_{i,l}^0 \prod_{j=1}^M t_j^{a_{ij}^0+a_{lj}^0}} \tag{$JCP_{dep}$} \label{JCP_dep}\\
&\text{s.t}&
        \sum_{i=1}^{I_k} \mu_i^k \prod_{j=1}^M t_j^{a_{ij}^k}+\sqrt{\gamma_1^k} \sqrt{\sum\limits_{i=1}^{ I_k}\sum\limits_{l=1}^{I_k} \sigma_{i,l}^k \prod_{j=1}^M t_j^{a_{ij}^k+a_{lj}^k}} \notag\\
        & & + \sqrt{\frac{y_k}{1-y_k}}\sqrt{\gamma_2^k} \sqrt{\sum\limits_{i=1}^{ I_k}\sum\limits_{l=1}^{I_k} \sigma_{i,l}^k \prod_{j=1}^M t_j^{a_{ij}^k+a_{lj}^k}} \leq 1 \quad \forall k=1,..., K \notag\\
&  &\sum_{k=1}^K y_k \geq K-\epsilon, \text{ }  0 < y_k \leq 1 \quad \forall k=1,...,K \notag
\end{eqnarray} 
As for the independent case, we obtain using the same logarithmic transformation the following biconvex equivalent problem for (\ref{JCP_dep})

\begin{eqnarray} 
& \min \limits_{r \in \mathbb{R}^M, x \in \mathbb{R}^K } &  \sum_{i =1}^{I_0} \mu_i^0 \text{exp}\left\{\sum_{j=1}^M {a_{ij}^0}r_j\right\} + \sqrt{\gamma_1^0} \sqrt{\sum\limits_{i =1}^{ I_0}\sum\limits_{l =1}^{I_0} \sigma_{i,l}^0  \text{exp} \left\{\sum_{j=1}^M (a_{ij}^0+a_{lj}^0)r_j\right\}} \tag{$JCP_{dep}^{log}$} \label{JCP_dep^log} \\
& \text{s.t}& 
        \sum_{i =1}^{I_k} \mu_i^k  \text{exp}\left\{\sum_{j=1}^M {a_{ij}^k}r_j\right\}+\sqrt{\gamma_1^k}  \sqrt{\sum\limits_{i =1}^{ I_k}\sum\limits_{l =1}^{I_k} \sigma_{i,l}^k  \text{exp} \left\{\sum_{j=1}^M (a_{ij}^k+a_{lj}^k)r_j\right\}} \notag\\
        & & + \sqrt{\gamma_2^k} \sqrt{\sum\limits_{i =1}^{ I_k}\sum\limits_{l =1}^{I_k} \sigma_{i,l}^k  \text{exp} \left\{\sum_{j=1}^M (a_{ij}^k+a_{lj}^k)r_j+\text{log} \left(\frac{y_k}{1-y_k}\right)\right\}} \leq 1 \quad \forall k=1,..., K \notag\\
&  &\sum_{k=1}^K y_k \geq K-\epsilon \quad 0 < y_k \leq 1 \quad \forall k=1,...,K \notag
\end{eqnarray}   
\begin{theorem} \cite{liu2022distributionally}
    If $\epsilon \leq 0.5$ and $\sigma_{i,l}^k \geq 0$ for all $i$, $l$ and $k$, problem (\ref{JCP_dep^log}) is a convex programming problem.
\end{theorem}
\subsection{Uncertainty Sets with Known First Order Moment and Nonnegative Support}
\noindent In this section, we consider uncertainty sets with nonnegative supports and known first-order moments. The uncertainty sets for (\ref{JCP}) can be formulated as follows
\begin{align*}
\mathcal{D}^3_k(\mu_k, \Sigma_k)=  \biggl \{\mathcal{F}_k \bigg|
\begin{array}{rcl} 
  \mathbb{E}[C^k] = \mu^k\\
  \mathbb{P}_{\mathcal{F}_k}[C^k \geq 0] = 1
\end{array}
 \biggl \} \quad \forall k=1,...,K
\end{align*}
where $\mu^k > 0$.
\subsubsection{Case \texorpdfstring{(\ref{JCP})}\ \ with Jointly Independent Row Vectors.}
\noindent We first consider the case when the marginal distributions in the uncertainty set are jointly independent. Using the strong duality \cite{liu2022distributionally}, (\ref{JCP}) can be reformulated as follows \cite{peng2022bounds}
\\
\begin{eqnarray}
&\min\limits_{t \in \mathbb{R}_{++}^M, \lambda, \beta, \pi}   &    \sum\limits_{i=1}^{I_0} \mu_i^0 \prod_{j=1}^M t_j^{a_{ij}^0} \tag{$JCP_{NS}^{ind}$} \label{JCP_NS^ind}\\
&{\text{s.t}}& \prod_{k=1}^K y_k \geq 1-\epsilon \quad \forall y_k \in [0, 1] \quad \forall k=1,...,K \notag\\
  & &       y_k \lambda_k^{-1} - \lambda_k^{-1}{\beta^k}^T\mu^k \leq 1 \quad \forall k=1, ...., K \notag\\
         & & \beta_k \leq 0, 0 < \lambda \leq 1 \quad \forall  k=1, ...., K \notag\\
         & & \lambda_k^{-1} \pi_k \geq 1 \quad \forall k=1, ...., K \notag\\
         & & (-\beta_k)^{-1} \pi_k  \prod_{j=1}^M t_j^{a_{ij}^k} \leq 1 \quad \forall i=1,...,I_k \quad \forall k=1, ...., K \notag
\end{eqnarray}
\\
(\ref{JCP}) can be reformulated as a convex problem using a logarithmic transformation $x_j = \text{log}(y_j)$, $t_j = \text{log}(r_j)$, $\Tilde{\lambda_k} = \text{log}(\lambda_k)$, $\Tilde{\beta_k} = \text{log}(-\beta_k)$, $\Tilde{\pi} = \text{log}(\pi)$. Problem (\ref{JCP_NS^ind}) becomes,
\begin{eqnarray}
&\min\limits_{x, r, \Tilde{\lambda}, \Tilde{\beta}, \Tilde{\pi}}  &    \sum\limits_{i=1}^{I_0} \mu_i^0 \text{exp}\left\{\sum_{j=1}^M a_{ij}^0 r_j\right\} \tag{$JCP_{NS-ind}^{log}$} \label{JCP_NS-ind^log} \\
& {\text{s.t}}& \sum\limits_{k=1}^K x_k \geq \text{log}(1-\epsilon) \quad \forall x_k \leq 0 \quad \forall k=1,...,K \notag\\
   & &      \text{exp}(x_k-\Tilde{\lambda}_k) + \sum_{i=1}^{I_k} \text{exp}\left\{-\Tilde{\lambda}_k+\Tilde{\beta}_i^k+\text{log}{\mu_i^k} \right\}\leq 1 \quad \forall k=1, ...., K \notag\\
         & & \Tilde{\lambda}_k \leq 0 \quad \forall k=1, ...., K \notag\\
         & & \Tilde{\lambda}_k \leq \Tilde{\pi}_k \quad \forall k=1, ...., K \notag\\
         & & \Tilde{\pi}_k +  \sum_{j=1}^M a_{ij}^kr_j - \Tilde{\beta}_i^k \leq 0, \quad \forall i=1,...,I_k \quad \forall k=1, ...., K \notag
\end{eqnarray}
\subsubsection{Case \texorpdfstring{(\ref{JCP})}\ \ with Jointly Dependent Row Vectors.}
\noindent In the case where the constraints of (\ref{JCP}) are jointly dependent, we have the following deterministic equivalent 
\begin{eqnarray}
& \min\limits_{t \in \mathbb{R}_{++}^M, \lambda, \beta, \pi}   &    \sum\limits_{i=1}^{I_0} \mu_i^0 \prod_{j=1}^M t_j^{a_{ij}^0}, \tag{$JCP_{NS-ind}^{log}$} \label{JCP_NS^dep}\\
&{\text{s.t}}& \sum_{k=1}^K y_k \geq K-\epsilon, \text{ } 0 \leq y_k \leq 1 \text{ , }  k=1,...,K, \notag\\
  & &       y_k \lambda_k^{-1} - \lambda_k^{-1}{\beta^k}^T\mu^k \leq1, k=1, ...., K, \notag\\
         & & \beta_k \leq 0, 0 < \lambda \leq 1, k=1, ...., K, \notag\\
         & & \lambda_k^{-1} \pi_k \geq 1,  k=1, ...., K, \notag\\
         & & (-\beta_k)^{-1} \pi_k  \prod_{j=1}^M t_j^{a_{ij}^k} \leq 1, i=1,...,I_k,  k=1, ...., K, \notag
\end{eqnarray}
We apply a log transformation to convert (\ref{JCP_NS^dep}) into a biconvex problem. We take $t_j = \text{log}(r_j)$, $\Tilde{\lambda_k} = \text{log}(\lambda_k)$, $\Tilde{\beta_k} = \text{log}(-\beta_k)$, $\Tilde{\pi} = \text{log}(\pi)$ and obtain
\begin{eqnarray}
&\min\limits_{x, r, \Tilde{\lambda}, \Tilde{\beta}, \Tilde{\pi}} & \sum\limits_{i=1}^{I_0} \mu_i^0 \text{exp}\left\{\sum_{j=1}^M a_{ij}^0 r_j\right\} \tag{$JCP_{NS-dep}^{log}$} \label{JCP_NS-dep^log}\\
& {\text{s.t}}& \sum_{k=1}^K y_k \geq K-\epsilon \quad \forall y_k \in [0,1] \quad \forall k=1,...,K \notag\\
   & & y_k \text{exp}(-\Tilde{\lambda}_k) + \sum_{i=1}^{I_k} \text{exp}\left\{-\Tilde{\lambda}_k+\Tilde{\beta}_i^k+\text{log}{\mu_i^k} \right\}\leq 1 \quad \forall k=1, ...., K \notag\\
         & & \Tilde{\lambda}_k \leq 0, \quad \forall k=1, ...., K \notag\\
         & & \Tilde{\lambda}_k \leq \Tilde{\pi}_k \quad \forall k=1, ...., K \notag\\
         & & \Tilde{\pi}_k +  \sum_{j=1}^M a_{ij}^kr_j - \Tilde{\beta}_i^k \leq 0 \quad \forall i=1,...,I_k \quad \forall k=1, ...., K \notag
\end{eqnarray}
\section{A dynamical recurrent neural network for
\texorpdfstring{(\ref{JCP_dep^log}), (\ref{JCP_ind^log})} \ \ and
(\texorpdfstring{\ref{JCP_NS-ind^log}} \ )}
\label{sec:single_recurrent_neural_network}

\noindent The proposed neural network
algorithm in this section is an important extension of the one in Tassouli \&
Lisser \cite{TASSOULI2023765}.\\
\\
Observe that (\ref{JCP_dep^log}), (\ref{JCP_ind^log}) and (\ref{JCP_NS-ind^log}) can be written in the following general form 
\begin{eqnarray}
&\min\limits_{z} & f(z), \label{eq:min_f}  \\
&{\rm s.t.}&  g(z) \leq 0, \notag
\end{eqnarray}
where $f$ and $g$ are two convex functions.
\\
For (\ref{JCP_ind^log}), $z=(r,x)^T$, $f(z) = \sum_{i =1}^{I_0} \mu_i^0 \text{exp}\left\{\sum_{j=1}^M {a_{ij}^0}r_j\right\} + \sqrt{\gamma_1^0} \sqrt{\sum\limits_{i =1}^{ I_0}\sum\limits_{l =1}^{I_0} \sigma_{i,l}^0  \text{exp} \left\{\sum_{j=1}^M (a_{ij}^0+a_{lj}^0)r_j\right\}}$ and 
\begin{align*}
g(z) = \left( \begin{array}{c}
  \sum\limits_{i =1}^{I_1} \mu_i^1  \text{exp}\left\{\sum_{j=1}^M {a_{ij}^1}r_j\right\}+\sqrt{\gamma_1^1}  \sqrt{\sum\limits_{i =1}^{ I_1}\sum\limits_{l =1}^{I_1} \sigma_{i,l}^1  \text{exp} \left\{\sum\limits_{j=1}^M (a_{ij}^1+a_{lj}^1)r_j\right\}} \\
  + \sqrt{\frac{e^{x_1}}{1-e^{x_1}}}\sqrt{\gamma_2^1} \sqrt{\sum\limits_{i =1}^{ I_1}\sum\limits_{l =1}^{I_1} \sigma_{i,l}^1  \text{exp} \left\{\sum\limits_{j=1}^M (a_{ij}^1+a_{lj}^1)r_j\right\}} -1\\
  \vdots \\
    \sum\limits_{i=1}^{I_K} \mu_i^K  \text{exp}\left\{\sum\limits_{j=1}^M {a_{ij}^K}r_j\right\}+\sqrt{\gamma_1^K}  \sqrt{\sum\limits_{i =1}^{ I_K}\sum\limits_{l =1}^{I_K} \sigma_{i,l}^K  \text{exp} \left\{\sum_{j=1}^M (a_{ij}^K+a_{lj}^K)r_j\right\}}+\\ \sqrt{\frac{e^{x_K}}{1-e^{x_K}}}\sqrt{\gamma_2^K} \sqrt{\sum\limits_{i =1}^{ I_K}\sum\limits_{l =1}^{I_K} \sigma_{i,l}^K  \text{exp} \left\{\sum\limits_{j=1}^M (a_{ij}^K+a_{lj}^K)r_j\right\}} -1\\
     \text{log}(1-\epsilon) - \sum_{k=1}^K x_k\\
     x_1\\
     \vdots\\
     x_K
\end{array} \right)
\end{align*}
\\
For (\ref{JCP_dep^log}), $z=(r,x)^T$, $f(z) = \sum_{i =1}^{I_0} \mu_i^0 \text{exp}\left\{\sum_{j=1}^M {a_{ij}^0}r_j\right\} + \sqrt{\gamma_1^0} \sqrt{\sum\limits_{i =1}^{ I_0}\sum\limits_{l =1}^{I_0} \sigma_{i,l}^0  \text{exp} \left\{\sum_{j=1}^M (a_{ij}^0+a_{lj}^0)r_j\right\}}$ and 
\begin{align*}
  g(z) = \left( \begin{array}{c}
  \sum\limits_{i =1}^{I_1} \mu_i^1  \text{exp}\left\{\sum_{j=1}^M {a_{ij}^1}r_j\right\}+\sqrt{\gamma_1^1}  \sqrt{\sum\limits_{i =1}^{ I_1}\sum\limits_{l =1}^{I_1} \sigma_{i,l}^1  \text{exp} \left\{\sum\limits_{j=1}^M (a_{ij}^1+a_{lj}^1)r_j\right\}} \\
  + \sqrt{\frac{e^{x_1}}{1-e^{x_1}}}\sqrt{\gamma_2^1} \sqrt{\sum\limits_{i =1}^{ I_1}\sum\limits_{l =1}^{I_1} \sigma_{i,l}^1  \text{exp} \left\{\sum\limits_{j=1}^M (a_{ij}^1+a_{lj}^1)r_j\right\}} -1\\
  \vdots \\
    \sum\limits_{i=1}^{I_K} \mu_i^K  \text{exp}\left\{\sum\limits_{j=1}^M {a_{ij}^K}r_j\right\}+\sqrt{\gamma_1^K}  \sqrt{\sum\limits_{i =1}^{ I_K}\sum\limits_{l =1}^{I_K} \sigma_{i,l}^K  \text{exp} \left\{\sum_{j=1}^M (a_{ij}^K+a_{lj}^K)r_j\right\}}+\\ \sqrt{\frac{e^{x_K}}{1-e^{x_K}}}\sqrt{\gamma_2^K} \sqrt{\sum\limits_{i =1}^{ I_K}\sum\limits_{l =1}^{I_K} \sigma_{i,l}^K  \text{exp} \left\{\sum\limits_{j=1}^M (a_{ij}^K+a_{lj}^K)r_j\right\}} -1\\
     \text{log}(K-\epsilon) - \sum_{k=1}^K x_k\\
     x_1\\
     \vdots\\
     x_K
\end{array} \right)
\end{align*}
\\
For (\ref{JCP_NS-ind^log}),  $z = (r,x, \Tilde{\lambda}, \Tilde{\beta}, \Tilde{\pi})^T$, $f(z) =\sum\limits_{i=1}^{I_0} \mu_i^0 \prod_{j=1}^M t_j^{a_{ij}^0} $ and \begin{align*}
    g(z) = \left( \begin{array}{c}
 \text{log}(1-\epsilon) -  \sum\limits_{k=1}^K x_k\\
 x_1\\
 \vdots\\
 x_K\\
  \text{exp}(x_1-\Tilde{\lambda}_1) + \sum\limits_{i=1}^{I_1} \text{exp}\left\{-\Tilde{\lambda}_1+\Tilde{\beta}_i^1+\text{log}{\mu_i^1} \right\}-1\\
  \vdots\\
   \text{exp}(x_K-\Tilde{\lambda}_K) + \sum\limits_{i=1}^{I_K} \text{exp}\left\{-\Tilde{\lambda}_K+\Tilde{\beta}_i^K+\text{log}{\mu_i^K} \right\} -1\\
    \Tilde{\lambda}_1\\
    \vdots\\
    \Tilde{\lambda}_1 -\Tilde{\pi}_1 \\
    \vdots \\
     \Tilde{\lambda}_K -\Tilde{\pi}_K \\
  \Tilde{\pi}_1 +  \sum_{j=1}^M a_{ij}^1r_j - \Tilde{\beta}_i^1\leq 0, i=1,...,I_1\\
  \vdots\\
   \Tilde{\pi}_K +  \sum_{j=1}^M a_{ij}^Kr_j - \Tilde{\beta}_i^K\leq 0, i=1,...,I_K\\
\end{array} \right)
\end{align*}
We know that $z^*$ is an optimal solution of (\ref{eq:min_f}) if and only if the following Karush–Kuhn–Tucker (KKT) conditions are satisfied.
\begin{align*}
    &\nabla f(z) + \nabla g(z)^T\gamma = 0\\
    & \gamma \geq 0, \gamma^Tg(z)=0&
\end{align*}
Now, let $z(.)$ and $\gamma(.)$ be two time-dependent variables. The objective is to develop a continuous-time dynamical system that converges to the KKT point of the nonlinear programming problem (\ref{eq:min_f}). Hence, our current goal is to construct a neural network capable of converging to the KKT point of the nonlinear programming problem (\ref{eq:min_f}). We can outline the neural network model associated with (\ref{eq:min_f}) by the following nonlinear dynamical system
\begin{align}
   & \kappa \frac{dz}{dt} = -\left( \nabla f(z)+ \nabla g(z)^T\left(\gamma+g(z)\right)_+ \right)&\label{eq:KKT_1}\\
   & \kappa \frac{d\gamma}{dt} = -\gamma+\left(\gamma+g(z)\right)_+&\label{eq:KKT_2}
\end{align}
where $\kappa$ is a given convergence rate and  {$(x)_+ = \left(\max(x_1,0), \max(x_2,0), ..., \max(x_n,0)\right)$, for each $x=(x_1,x_2, ..., x_n) \in \mathbb{R}^n$}.

\begin{lemma}\label{th2}
    $(z^*, \gamma^*)$ is an equilibrium point of (\ref{eq:KKT_1})-(\ref{eq:KKT_2}) if and only if $z^*$ is an optimal solution of (\ref{eq:min_f}) where $\gamma^*$ is the corresponding Lagrange multiplier.
\end{lemma}

\begin{lemma}
For any initial point $(z(t_0), \gamma(t_0))$, there exists
a unique continuous solution $(z(t), \gamma(t))$ for (\ref{eq:KKT_1})-(\ref{eq:KKT_2}).
\end{lemma}

\begin{lemma}
    The neural network proposed in equations (\ref{eq:KKT_1})-(\ref{eq:KKT_2}) exhibits global stability in the Lyapunov sense. Furthermore, the dynamical network globally converges to a KKT point denoted $(z^*, \gamma^*)$ where $z^*$ is the optimal solution of the problem (\ref{eq:min_f}).
\end{lemma}

\begin{proof}
    Let $\zeta = (z, \gamma)$, we define 
    \begin{align*}
        U(\zeta) =  \begin{bmatrix}
     -\left( \nabla f(z)+ \nabla g(z)^T\left(\gamma+g(z)\right)_+ \right)\\
    -\gamma+\left(\gamma+g(z)\right)_+
  \end{bmatrix}
  \end{align*}
  First, consider the following Lyapunov function
  \begin{equation*}
      E(\zeta) = ||U(\zeta)||^2 + \frac{1}{2} ||\zeta-\zeta^*||, 
  \end{equation*} where $\zeta^* = (z^*, \gamma^*)$ is an equilibrium point of (\ref{eq:KKT_1})-(\ref{eq:KKT_2}).\\
  $\frac{dE(\zeta(t))}{dt} = \frac{dU}{dt}^T U + U^T\frac{dU}{dt}+ (\zeta-\zeta^*)^T\frac{d\zeta}{dt}$. Observe that $\frac{dU}{dt}= \frac{dU}{d\zeta} \times \frac{d\zeta}{dt} = \nabla U(\zeta)U(\zeta)$.
  Without loss of generality suppose that there exists $p\in \mathbb{N}$ such that $(\gamma+g(z))_+ = (\gamma_1+g_1(z)), ..., (\gamma_p+g_p(z)), 0, ..., 0)$, and we define $g^p = (g_1, ..., g_p)$.\\
  We have
  \begin{align*}
  \nabla U(\zeta) =  \begin{bmatrix}
     -\left( \nabla^2 f(z)+ \sum_{i=1}^p\nabla^2 g^p(z)\left(\gamma_p+g_p(z)\right)+\nabla g(z)^T\nabla g(z) \right) & -\nabla g^p(z)^T\\
    \nabla g^p(z) & S_p
  \end{bmatrix}
  \end{align*}
  where $S_p  =\begin{bmatrix}
&O_{p \times p} & O_{p \times (N-p)}\\
&O_{(N-p) \times q} & I_{(N-p) \times (N-p)}\\
\end{bmatrix}$ and $N$ is the length of vector $\gamma$.\\
Since $f$ and $g$ are convex, then the Hessian matrices $\nabla^2f$ and $\nabla^2g^p$ are positive semidefinite. Furthermore $\nabla g^T\nabla g$ is positive semidefinite, we conclude that $\nabla U$ is negative semidefinite.\\
\\
Back to the expression of $\frac{dE(\zeta(t))}{dt}$, we have

\begin{align*}
    \frac{dE(\zeta(t))}{dt} = \underbrace{U^T(\nabla U+ \nabla U^T) U}_{\leq 0 \text{ since $\nabla U$ is negative semidefinite}}+  \underbrace{(\zeta-\zeta^*)^T (U(\zeta)-U(\zeta^*))}_{\leq 0 \text{ by Lemma 4 in \cite{TASSOULI2023765}}} \leq 0
\end{align*}
\\
Then, the neural network (\ref{eq:KKT_1})-(\ref{eq:KKT_2}) is globally stable in the sense of Lyapunov. Next, similarly to the proof of Theorem 5 in \cite{TASSOULI2023765}, we prove that the dynamical neural network (\ref{eq:KKT_1})-(\ref{eq:KKT_2}) is globally convergent to $(z^*, \gamma^*)$ where $z^*$ is the optimal solution of (\ref{eq:min_f}).
\end{proof}

\noindent The advantage of a dynamical recurrent neural network to solve
problems (\ref{JCP_dep^log}), (\ref{JCP_ind^log}) and (\ref{JCP_NS-ind^log}) is
to solve multiple instances of the problems at the same time after training the
neural network. To do this, we consider a conditional recurrent neural network
architecture \cite{lauria2024conditional}. This variant of RNN conditions the
prediction based on a given set of external parameters $\theta \in \Theta$, which
are not part of the sequential data. We consider problem data as external
parameters modelling different instances of the same problem.\\
Let $\theta \in \Theta$ with $\Theta$ a compact set of parameters. The
conditional RNN associated to equations (\ref{eq:KKT_1})-(\ref{eq:KKT_2}) is
given by the following system (\ref{eq:conditional_KKT_1})-(\ref{eq:conditional_KKT_2})
\begin{align}
    & \kappa \frac{dz}{dt} = -\left( \nabla f_{\theta}(z)+ \nabla g_{\theta}(z)^T\left(\gamma+g_{\theta}(z)\right)_+ \right)&\label{eq:conditional_KKT_1}\\
    & \kappa \frac{d\gamma}{dt} =
-\gamma+\left(\gamma+g_{\theta}(z)\right)_+&\label{eq:conditional_KKT_2}
\end{align}

\noindent To solve multiple instances of the same problem, we consider that
each $\theta$ represents a different instance of the problem data
\cite{wu2023deep}. In the case of problems (\ref{JCP_dep^log}),
(\ref{JCP_ind^log}) and (\ref{JCP_NS-ind^log}), we have $\theta = \{(\mu_i^k,
a_i^k, \epsilon) \text{ } | \text{ } \forall k \in \llbracket 0, K \rrbracket
\text{ } \forall i \in I_k\}$. Each $\theta$ can also represent a combination of
such data. When solving one instance of optimization problem, the parameter
$\theta$ is constant, and the recurrent neural network
(\ref{eq:conditional_KKT_1})-(\ref{eq:conditional_KKT_2}) degrades to
(\ref{eq:KKT_1})-(\ref{eq:KKT_2}). To illustrate this feature, we present in
Figure \ref{fig:block_diagram_circuit} a generalized circuit implementation
realizing the system (\ref{eq:conditional_KKT_1})-(\ref{eq:conditional_KKT_2}).
\begin{figure}[!htb]
 \includegraphics[scale=1]{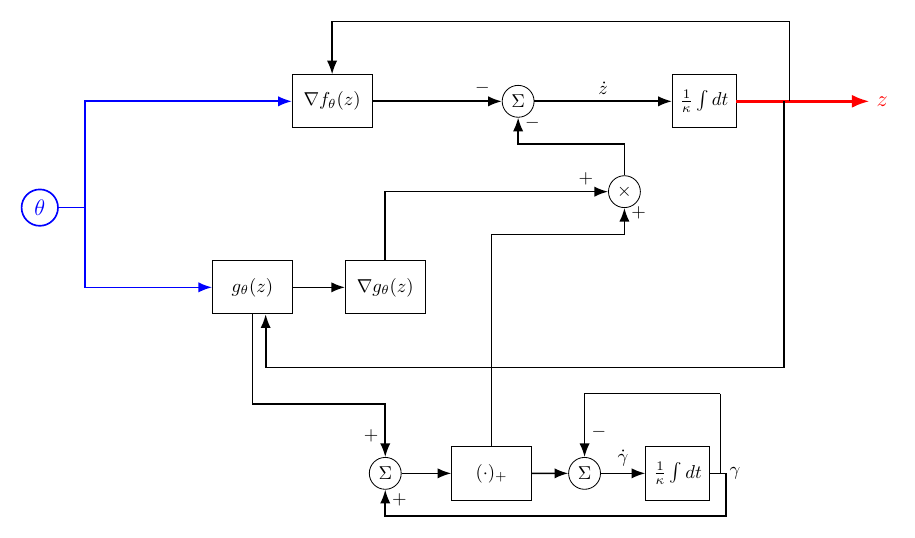}
 \caption{A block diagram for the neural network (\ref{eq:conditional_KKT_1}-\ref{eq:conditional_KKT_2})}
 \label{fig:block_diagram_circuit}
\end{figure}


\section{A two-time scale neurodynamic duplex for (\texorpdfstring{\ref{JCP_NS-dep^log}} \ )}
\label{sec:two-time_scale_neurodynamic_duplex}
(\ref{JCP_NS-dep^log}) can be written in the following general form 
\begin{eqnarray}
&\min\limits_{z,y} & f(z), \label{eq:min_f_z_y} \\
&{\rm s.t.}&  g(z,y) \leq 0, \notag
\end{eqnarray}
where $f$ is a convex function and $g$ is a biconvex function, $z = (r, \Tilde{\lambda}, \Tilde{\beta}, \Tilde{\pi})^T$, $f(z) =\sum\limits_{i=1}^{I_0} \mu_i^0 \prod_{j=1}^M t_j^{a_{ij}^0} $ and
\begin{align*}
g(z,y) = \left( \begin{array}{c}
K-\epsilon-\prod\limits_{k=1}^K y_k\\
 -y_1\\
 \vdots\\
 -y_K\\
  y_1-1\\
 \vdots\\
 y_K-1\\
  y_1\text{exp}(-\Tilde{\lambda}_1) + \sum\limits_{i=1}^{I_1} \text{exp}\left\{-\Tilde{\lambda}_1+\Tilde{\beta}_i^1+\text{log}{\mu_i^1} \right\}-1\\
  \vdots\\
  y_K \text{exp}(-\Tilde{\lambda}_K) + \sum\limits_{i=1}^{I_K} \text{exp}\left\{-\Tilde{\lambda}_K+\Tilde{\beta}_i^K+\text{log}{\mu_i^K} \right\} -1\\
    \Tilde{\lambda}_1\\
    \vdots\\
    \Tilde{\lambda}_1 -\Tilde{\pi}_1 \\
    \vdots \\
     \Tilde{\lambda}_K -\Tilde{\pi}_K \\
  \Tilde{\pi}_1 +  \sum_{j=1}^M a_{ij}^1r_j - \Tilde{\beta}_i^1\leq 0, i=1,...,I_1\\
  \vdots\\
   \Tilde{\pi}_K +  \sum_{j=1}^M a_{ij}^Kr_j - \Tilde{\beta}_i^K\leq 0, i=1,...,I_K\\
\end{array} \right)
\end{align*}
We denote $\mathcal{U} = \{z,y \text{ }| \text{ } g(z,y) \leq 0\}$ the feasible set of (\ref{eq:min_f_z_y}). The Lagrangian function of problem (\ref{eq:min_f_z_y}) is defined as follows: 
\begin{align}
\mathcal{L}(z,y,\omega) = f(z) + \omega^T g(z,y).
\end{align}
For any $(z,y) \in \mathcal{U}$, the KKT conditions are stated as follows:
\begin{align}
&\nabla L(z,y,\omega) = 0,\label{eq:KKT_Lagrangian_1}\\
&\omega \geq 0, \quad \omega^Tg(z,y) = 0.\label{eq:KKT_Lagrangian_2}
\end{align}

\begin{definition}\label{def1}
    Let $(z,y) \in \mathcal{U}$, $(z,y)$ is called a partial optimum of (\ref{eq:min_f_z_y}) if and only if 
    \begin{equation*}
        f(z) \leq f(\tilde{z}), \forall \tilde{z} \in \mathcal{U}_y,
    \end{equation*}
    where $\mathcal{U}_y = \{z \text{ }| \text{ } g(z,y) \leq 0\}$.
\end{definition}

\begin{lemma}\label{th:KKT_equivalence}
    The KKT system (\ref{eq:KKT_Lagrangian_1})-(\ref{eq:KKT_Lagrangian_2}) is equivalent to the following system
    \begin{align}
    &\nabla f(z) +  \nabla_z g(z, y)^T (\omega+g(x,z))_+ = 0 &\label{eq:KKT_equivalent_1}\\
    &  \nabla_yg(z, y)^T(\omega+g(z,y))_+ =0&\\
    &(\omega+g(x, z))_+-\omega=0 & \label{eq:KKT_equivalent_2}
\end{align}
\end{lemma}

\noindent Based on the equations (\ref{eq:KKT_equivalent_1})-(\ref{eq:KKT_equivalent_2}), we consider the following two-time-scale recurrent neural network model

\begin{align}
&\kappa_1 \frac{dz}{dt} =-\left( \nabla f(z) +  \nabla_z g(z, y)^T (\omega+g(x,z))_+\right),& \label{NN1}\\
&\kappa_2 \frac{dy}{dt} = -\left(\nabla_yg(z, y)^T(\omega+g(z,y))_+ \right),&\\
&\kappa_2 \frac{d\omega}{dt} = -\omega + (\lambda+g(z,y))_+ ,&\label{NN2}
\end{align}
\\
where $(z,y, \omega)$ are now time-dependent variables and $\kappa_1$ and
$\kappa_2$ are two time scaling constants with $\kappa_1 \neq \kappa_2$. We
propose a duplex of two two-time-scale recurrent neural network
(\ref{NN1})-(\ref{NN2}) for solving (\ref{eq:min_f}) one with $\kappa_1 >
\kappa_2$ and the second with $\kappa_1 < \kappa_2$ as shown in Figure
\ref{fig:block_diagram}. \FloatBarrier

 \begin{figure}[!htb] 
\hspace{2.8 cm}
 \includegraphics[scale=0.8]{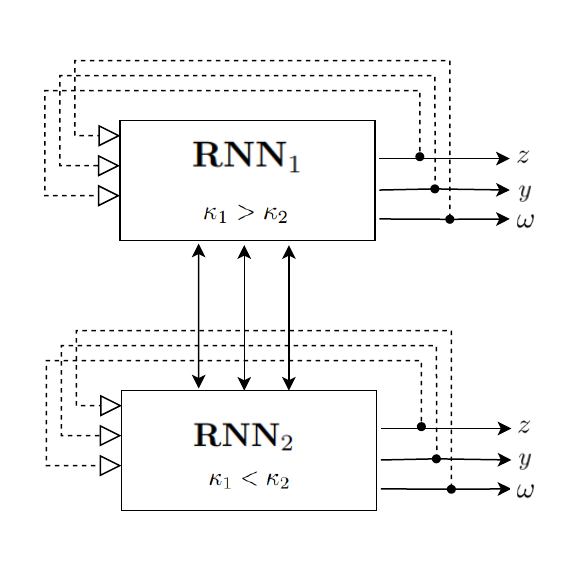}
 \caption{A block diagram depicting a duplex neurodynamic system with a two-timescale configuration}
\label{fig:block_diagram}
\end{figure}

\FloatBarrier

\begin{theorem}
    $(z,y,\omega)$ is an equilibrium point of (\ref{NN1})-(\ref{NN2}) if and only if $(z,y,\omega)$ is a KKT point of (\ref{eq:min_f_z_y}).
\end{theorem}
\begin{proof}
    Let $(z,y,\omega)$ is an equilibrium point of (\ref{NN1})-(\ref{NN2}). We have then 
    \begin{align*}
        \frac{dz}{dt} &=0  \Longleftrightarrow -\left( \nabla f(z) +  \nabla_z g(z, y)^T (\omega+g(x,z))_+\right) = 0\\
        \frac{dy}{dt} &=0   \Longleftrightarrow -\left(\nabla_yg(z, y)^T(\omega+g(z,y))_+ \right)=0\\
        \frac{d\omega}{dt} &=0  \Longleftrightarrow -\omega + (\lambda+g(z,y))_+=0.
    \end{align*}
    We obtain system (\ref{eq:KKT_equivalent_1})-(\ref{eq:KKT_equivalent_2}). By Theorem \ref{th:KKT_equivalence}, the conclusion follows. The converse part of the Theorem is straightforward.
\end{proof}

\noindent The process begins by initializing the state variables of the neurodynamic models. Subsequently, each model undergoes a precise local search based on its dynamics to optimize its performance. Once all neurodynamic models have converged to their equilibria, the initial states of the recurrent neural networks are optimized using the particle swarm optimization (PSO) updating rule. In this context, we represent the position of the $i^{th}$ particle as $\Lambda_i = (\Lambda_{i1}, ..., \Lambda_{in})^T$, and its velocity as $v_i = (v_{i1}, ..., v_{in})^T$. The inertia weight $w \in [0, 1]$ determines the extent to which the particle retains its previous velocity. The best previous position that yielded the maximum fitness value for the $i^{th}$ particle is denoted as $\tilde{\Lambda}_i = (\tilde{\Lambda}_{i1}, ..., \tilde{\Lambda}_{in})^T$, and the best position in the entire swarm that yielded the maximum fitness value is represented by $\hat{\Lambda} = (\hat{\Lambda}_1, ..., \hat{\Lambda}_n)^T$. The initial state of each neurodynamic model is updated using the PSO updating rule, as described in reference \cite{clerc2002particle}.
\begin{align}
   &v_i(j+1) = wv_i(j)+c_1r_1(\tilde{\Lambda}_i-\Lambda_i(j))+c_2r_2(\hat{\Lambda}_i-\Lambda_i(j)) \label{eq:PSO_update_1}\\
    &\Lambda_i(j+1)=\Lambda_i(j)+v_i(j+1).\label{eq:PSO_update_2}
\end{align}
where the iterative index is represented by $j$, while the two weighting parameters are denoted as $c_1$ and $c_2$ and $r_1$ and $r_2$ represent two random values from the interval $[0, 1]$.\\
\\
To achieve global convergence, the diversity of initial neuronal states is crucial. One approach to enhance this diversity is by introducing a mutation operator, which generates a random $\Lambda_i(j + 1)$. This random generation of $\Lambda_i(j + 1)$ helps increase the variation among the initial neuronal states. To measure the diversity of these states, we employ the following function
\begin{equation*}
    d = \frac{1}{n} \sum_{i=1}^n \|\Lambda_i(j+1)-\hat{\Lambda}(j) \|.
\end{equation*}
We utilize the wavelet mutation operator proposed in \cite{ling2008hybrid}, which is performed for the $i$-th particle if $d < \zeta$. The mutation operation is carried out as follows 
\begin{equation}
    \Lambda_i(j+1) = \left\{
    \begin{array}{ll}
        \Lambda_i(j)+\mu(h_i-\Lambda_i(j))  &\mbox{for } \mu > 0 \\
         \Lambda_i(j)+\mu(\Lambda_i(j)-l_i) &\mbox{for } \mu < 0 
    \end{array}
\right.\label{eq:wavelet_mutation}
\end{equation}
where $h_i$ and $l_i$ are the upper and the lower bounds for $\Lambda_i$, respectively. $\zeta >0$ is a given threshold and $\mu$ is defined using a wavelet function 
\begin{equation}
    \mu = \frac{1}{\sqrt{a}}e^{-\frac{\phi}{2a}} \text{cos}(5\frac{\phi}{a})
\end{equation}
\\
When the value of $\mu$ goes to 1, the mutated element of the particle moves towards the maximum value of $\Lambda_i(j+1)$. On the other hand, as $\mu$ approaches -1, the mutated element moves towards the minimum value of  {$\Lambda_i(j+1)$}. The magnitude of $|\mu|$ determines the size of the search space for $\Lambda_i(j+1)$, with larger values indicating a wider search space. Conversely, smaller values of $|\mu|$ result in a smaller search space, allowing for fine-tuning.\\
\\
To achieve fine-tuning, the dilation parameter $a$ is adjusted based on the current iteration $j$ relative to the total number of iterations $T$. Specifically, $a$ is set as a function of $j/T$, with $a = e^{10\frac{j}{T}}$. Additionally, $\phi$ is randomly generated from the interval $[-2.5a, 2.5a]$.\\
\\
The algorithm details are given in Algorithm \ref{alg:neurodynamic_duplex} where $\Lambda = (z,y, \omega)$

\begin{algorithm}
\caption{The neurodynamic duplex}\label{alg:neurodynamic_duplex}
\begin{algorithmic}
 \Initialize{- Let $\Lambda_1(0)$ and $\Lambda_2(0)$ randomly in the feasible region. \\- Let the initial best previous position and best position $\tilde{\Lambda}(0)= \hat{\Lambda}(0) = y=\Lambda(0)$.\\- Set the convergence error $\zeta$. \\ }
\\
\While{$||\Lambda(j+1)-\Lambda(j)|| \geq \epsilon$}\\
{\hspace{0.6cm} Compute the equilibrium points $\Bar{\Lambda}_1(j)$ and $\Bar{\Lambda}_2(j)$ of $\text{RNN}_1$ and $\text{RNN}_2$.}
\If{$f(\Bar{z}_1(j)) < f(\tilde{z}(j))$}
    \State $\tilde{\Lambda}(j+1) = \bar{\Lambda}_1(j)$
\Else
    \State $\tilde{\Lambda}(j+1) = \tilde{\Lambda}(j)$
\EndIf
\If{$f(\Bar{z}_2(j)) < f(\tilde{z}(j))$}
    \State $\tilde{\Lambda}(j+1) = \bar{\Lambda}_2(j)$
\Else
    \State $\tilde{\Lambda}(j+1) = \tilde{\Lambda}(j)$
\EndIf
\If{$f(\tilde{z}(j)) < f(\hat{z}(j))$}
    \State $\hat{\Lambda}(j+1) = \tilde{\Lambda}(j+1)$
\Else
    \State $\hat{\Lambda}(j+1) = \hat{\Lambda}(j)$
\EndIf\\
{\hspace{0.6cm}Compute the value of $\Lambda(j+1)$ following (\ref{eq:PSO_update_1})-(\ref{eq:PSO_update_2}).}
\If{$d < \zeta$}
    \State {Perform the wavelet mutation (\ref{eq:wavelet_mutation})}.
\EndIf\\
{\hspace{0.6cm}j=j+1}
\EndWhile
\end{algorithmic}
\end{algorithm}

\begin{lemma} \cite{uryasev2013stochastic} \label{lemma1}
   Suppose that the objective function $f$ is measurable, and the feasible region $\mathcal{U}$ is a measurable subset, and for any Borel subset $\mathcal{B}$ of $\mathcal{U}$ with positive Lebesgue measure we have $ \prod\limits_{k=1}^{\infty} (1-\mathbb{P}_k(\mathcal{B}))= 0$. Let $\{y(k)\}_{k=1}^{\infty}$ be a sequence generated by a stochastic optimization algorithm. If $\{y(k)\}_{k=1}^{\infty}$ is a nonincreasing sequence, then it converges in probability to the global optimum set.
\end{lemma}

\begin{theorem}\label{th:convergence_two_timescales}
    If the state of the neurodynamic model with a single timescale, described by the following equations
    \begin{align}
&\kappa \frac{dz}{dt} =-\left( \nabla f(z) +  \nabla_z g(z, y)^T (\omega+g(x,z))_+\right) \label{NN_single_timescale_1}\\
&\kappa \frac{dy}{dt} = -\left(\nabla_yg(z, y)^T(\omega+g(z,y))_+ \right),&\\
&\kappa \frac{d\omega}{dt} = -\omega + (\lambda+g(z,y))_+ \label{NN_single_timescale_2}
\end{align}
     converges to an equilibrium point, then the state of the neurodynamic model with two timescales, as described by equations (\ref{NN1})-(\ref{NN2}), globally converges to a partial optimum of problem (\ref{eq:min_f_z_y}).
\end{theorem}

\begin{proof}
We recall the Lagrangian function of (\ref{eq:min_f_z_y})

\begin{equation*}
    \mathcal{L}(z,y,\omega) = f(z) + \omega^Tg(z,y)
\end{equation*}
An equilibrium point $(z^*,y^*,\omega^*)$ of (\ref{NN_single_timescale_1})-(\ref{NN_single_timescale_2}) corresponds to a KKT point of (\ref{eq:min_f_z_y}). We fix $y^*$, and take $z \in \mathcal{U}_{y^*}$, (\ref{eq:min_f_z_y}) becomes a convex optimization problem and we have

\begin{equation*}
    \mathcal{L}(z^*,y^*,\omega^*) \leq \mathcal{L}(z,y^*,\omega^*)
\end{equation*}
which is equivalent to

\begin{equation*}
    f(z^*) + {\omega^*}^Tg(z^*,y^*) \leq f(z) + {\omega^*}^Tg(z,y^*)
\end{equation*}
As ${\omega^*}^Tg(z,y^*) \leq {\omega^*}^Tg(z^*,y^*) = 0$, we have $f(z^*) \leq f(z)$. By Definition \ref{def1}, $(z^*,y^*)$ is a partial optimum of \ref{eq:min_f_z_y}.
\end{proof}

\begin{theorem}
The duplex of two two-timescale neural networks in Figure \ref{fig:block_diagram} is globally convergent to a global optimal solution of problem (\ref{eq:min_f}).
\end{theorem}

\begin{proof}
By Theorem \ref{th:convergence_two_timescales}, the two-timescale neurodynamic models $\text{RNN}_1$ and $\text{RNN}_2$ are proven to converge to a partial optimum. From Algorithm \ref{alg:neurodynamic_duplex}, the solution sequence is generated as follows

\begin{align*}
\begin{cases}
\hat{\Lambda}(j+1) = \tilde{\Lambda}(j+1) \text{ if } f(\tilde{z}(j)) < f(\hat{z}(j))\\
\hat{\Lambda}(j+1) = \hat{\Lambda}(j) \text{ otherwise }
\end{cases}
\end{align*}
We observe that the generated solution sequence is monotonically increasing $\{f(\tilde{\Lambda}(j)) \}_{j=1}^{\infty}$. 
Let $\mathcal{M}_{i,j}$ represent the supporting set of the initial state of $\text{RNN}i$ at iteration $j$. According to equation (\ref{eq:wavelet_mutation}), the mutation operation ensures that the initial states of the recurrent neural networks are constrained to the feasible region $\mathcal{U}$. Therefore, for every iteration index $J \geq 1$, the supporting sets satisfy the following condition:
\begin{equation*}
    \mathcal{U} \subseteq \mathcal{M} = \bigcup_{j=1}^J \bigcup_{i=1}^2 \mathcal{M}_{i,j}.
\end{equation*}
Consequently, we have $v(\mathcal{U}) = v(\mathcal{M}) > 0$.\\
By Lemma \ref{lemma1}, we have 
\begin{equation*}
    \lim_{j -> \infty} \mathbb{P}({\hat{\Lambda}(j) \in \Phi}) = 1
\end{equation*}
where $\Phi$ is the set of the global optimal solutions of (\ref{eq:min_f}). The conclusion follows.
\end{proof}

\noindent Similarly as developed in the previous Section
\ref{sec:single_recurrent_neural_network}, the neurodynamic duplex allows
solving multiple instances of optimization problems of the form
(\ref{JCP_NS-dep^log}) by considering conditional RNNs for $\text{RNN}_1$ and
$\text{RNN}_2$. Let $\theta \in \Theta$ be the parameter representing the
problem data, as defined for the single recurrent neural network described by
equations (\ref{eq:KKT_1})-(\ref{eq:KKT_2}). We define the duplex of recurrent
neural networks as in equations (\ref{conditional_NN1})-(\ref{conditional_NN2})
\begin{flalign}
    &\kappa_1 \frac{dz}{dt} =-\left( \nabla f_{\theta}(z) + \nabla_z g_{\theta}(z, y)^T (\omega+g_{\theta}(x,z))_+\right),& \label{conditional_NN1}\\
    &\kappa_2 \frac{dy}{dt} = -\left(\nabla_y g_{\theta}(z, y)^T(\omega + g_{\theta}(z,y))_+ \right),&\\
    &\kappa_2 \frac{d\omega}{dt} = -\omega + (\lambda+g_{\theta}(z,y))_+ ,&\label{conditional_NN2}
\end{flalign}
The flowchart of the neurodynamic duplex, as described in the Algorithm \ref{alg:neurodynamic_duplex}, using conditional RNNs is presented in Figure
\ref{fig:block_diagram_flowchart_PSO}.
\begin{figure}[!htb]
 \hspace{0.75cm}
 \includegraphics[scale=0.75]{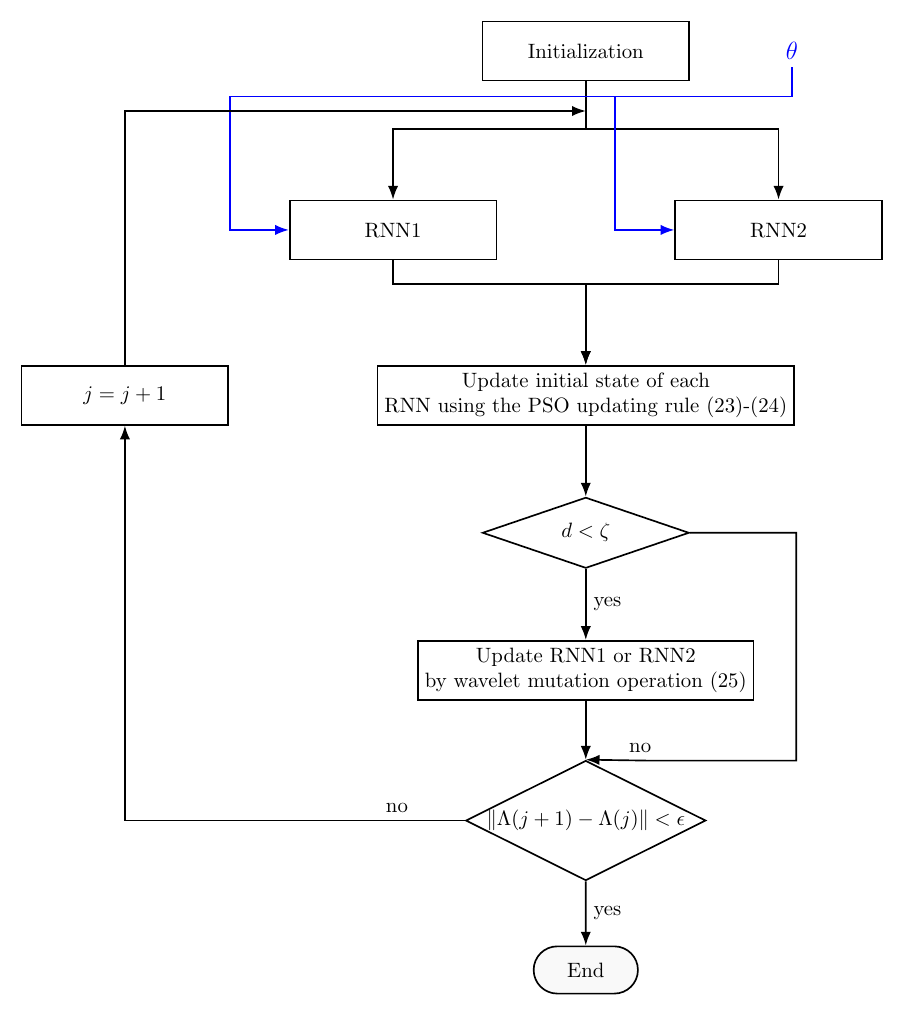}
 \caption{A block diagram of the neurodynamic duplex for the neural network (\ref{conditional_NN1})-(\ref{conditional_NN2})}
\label{fig:block_diagram_flowchart_PSO}
\end{figure}

\section{Numerical experiments}
\label{sec:numerical_experiments}

\noindent We consider three geometric optimization problems to evaluate the
performance of our neurodynamic approaches. All the algorithms in this Section
are implemented in Python. We run our algorithms on Intel(R) Core(TM) i7-10610U
CPU @ 1.80GHz. The random instances are generated with numpy.random, and we
solve the ODE systems with solve\_ivp of scipy.integrate. The deterministic
equivalent programs are solved with the package GEKKO \cite{beal2018gekko} and
the gradients and partial derivatives are computed with autograd.grad and
autograd.jacobian. For the following numerical experiments, we set
$\gamma^k_{1}=2$,  $\gamma^k_{2}=2$ and the error tolerance for the neurodynamic
duplex $\zeta = 10^{-4}$. In Subsection \ref{subsec:numerical_exp_2}, we
evaluate the quality of our neurodynamic duplex by comparing the obtained
solutions with the ones given by the Convex Alternate Search (CAR) from
\cite{biconvex}. The gap between the two solutions is computed as follows GAP =
$\frac{\text{Sol}_{\text{CAR}} -
\text{Sol}_{\text{Duplex}}}{\text{Sol}_{\text{CAR}}}$, where
$\text{Sol}_{\text{CAR}}$ and $\text{Sol}_{\text{Duplex}}$ are the solutions
obtained using the CAR and the neurodynamic duplex, respectively. For the
neurodynamical duplex, we take $\frac{\kappa_1}{\kappa_2}=0.1$ for the first
dynamical neural network and $\frac{\kappa_1}{\kappa_2}=10.0$ for the second
one. We introduce conditional RNNs in the neurodynamic duplex using
\textit{cond\_rnn} library \url{https://github.com/philipperemy/cond_rnn} to
solve multiple instances of the problem. To evaluate the robustness of our
approaches, we generate a set of 100 out-of-sample random scenarios of the
stochastic constraints for each problem, and we visualize the number of
scenarios for which the constraints were not respected. We call such scenarios
violated scenarios (VS).
\subsection{Uncertainty Sets with First Two Order Moments}
\subsubsection{A three-dimension shape optimization problem}
\noindent We first consider a transportation problem involving the shifting of
grain from a warehouse to a factory. The grain is transported within an open
rectangular box, with dimensions of length $x_1$ meters, width $x_2$ meters, and
height $x_3$ meters, as illustrated in Figure \ref{fig1}. The objective of the
problem is to maximize the volume of the rectangular box, given by the product
of its length, width, and height ($x_1 x_2 x_3$). However, two constraints must
be satisfied. The first constraint relates to the floor area of the box, and the
second constraint relates to the wall area. These constraints are necessary to
ensure that the shape of the box aligns with the requirements of a given truck.
In our analysis, we assume that the wall area $A_{wall}$ and the floor area
$A_{floor}$ are random variables.
\begin{figure}[!htb]
\centering
\includegraphics[width=0.55\columnwidth]{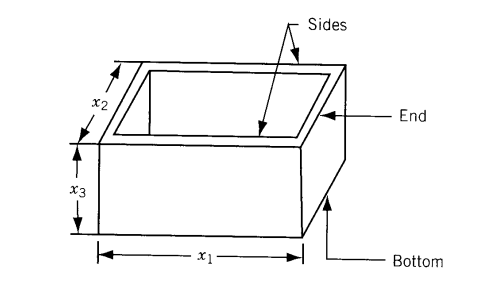}
\caption{3D-box shape \cite{rao2009geometric}}
\label{fig1}
\end{figure}
We formulate our shape optimization problem as follows 
\begin{center}
\begin{eqnarray}
&\min\limits_{x \in {\mathbb{R}^3_{++}}} & x_1^{-1} x_2^{-1}x_3^{-1}, \label{eq:shape_optimization_pb_1} \\
&\text{s.t}& 
       \inf\limits_{\mathcal{F} \in \mathcal{D} }\mathbb{P}_{\mathcal{F}}\left(  \frac{1}{A_{wall}} (2x_3x_2 + 2x_1x_3) \leq 1, \frac{1}{A_{floor}}x_1x_2\leq 1\right)  \geq 1-\epsilon. \notag
\end{eqnarray}
\end{center}
where $\mathcal{F}$ is the joint distribution for $\frac{1}{A_{wall}}$ and
$\frac{1}{A_{floor}}$ and $\mathcal{D}$ is the uncertainty set for the
probability distribution $\mathcal{F}$. We solve problem
(\ref{eq:shape_optimization_pb_1}) when the uncertainty set is equal to
$\mathcal{D}^2$ using the dynamical neural network
(\ref{eq:KKT_1})-(\ref{eq:KKT_2}). For the numerical experiments, we take the
mean and the covariance describing the uncertainty sets for $\frac{1}{A_{wall}}$
$m_{wall} =0.05 $, $\sigma_{wall}=0.01$, respectively and for
$\frac{1}{A_{floor}}$ $m_{floor} =0.5 $, $\sigma_{floor}=0.1$,  respectively.
We recapitulate the obtained results in Table \ref{table1}. Columns 1, 2 and 3
give the optimal value, the CPU time and the number of violated scenarios (VS)
in the independent case, respectively. Columns four, five and six show the
optimal value, the CPU time and the number VS in the dependent case,
respectively. The dynamic neural network covers well the risk region in both
cases. Figure \ref{fig:convergence_shape_optimization_problem} show the
convergence of the state variables.

\begin{table}
 \hspace{1.8cm}   \begin{tabular}{ccccccccccccc}
  \firsthline
 \multicolumn{3}{c}{Independent case} & & \multicolumn{3}{c}{ Dependent case}  \\
\cline{1-3} \cline{5-7} 
    Obj value& CPU Time & VS & & Obj value &CPU Time & VS   \\ \hline
    $0.296$  &0.43 &0&  & $0.298$&0.46 & 0   \\
\hline
\end{tabular}
\caption{{Results of solving problem (\ref{eq:shape_optimization_pb_1}) when $\mathcal{D}=\mathcal{D}^2$}} \label{table1}
\end{table}
\begin{figure}
 \hspace{-0.5 cm}\begin{subfigure}{.5\textwidth}
  \centering
 \includegraphics[scale = 0.4]{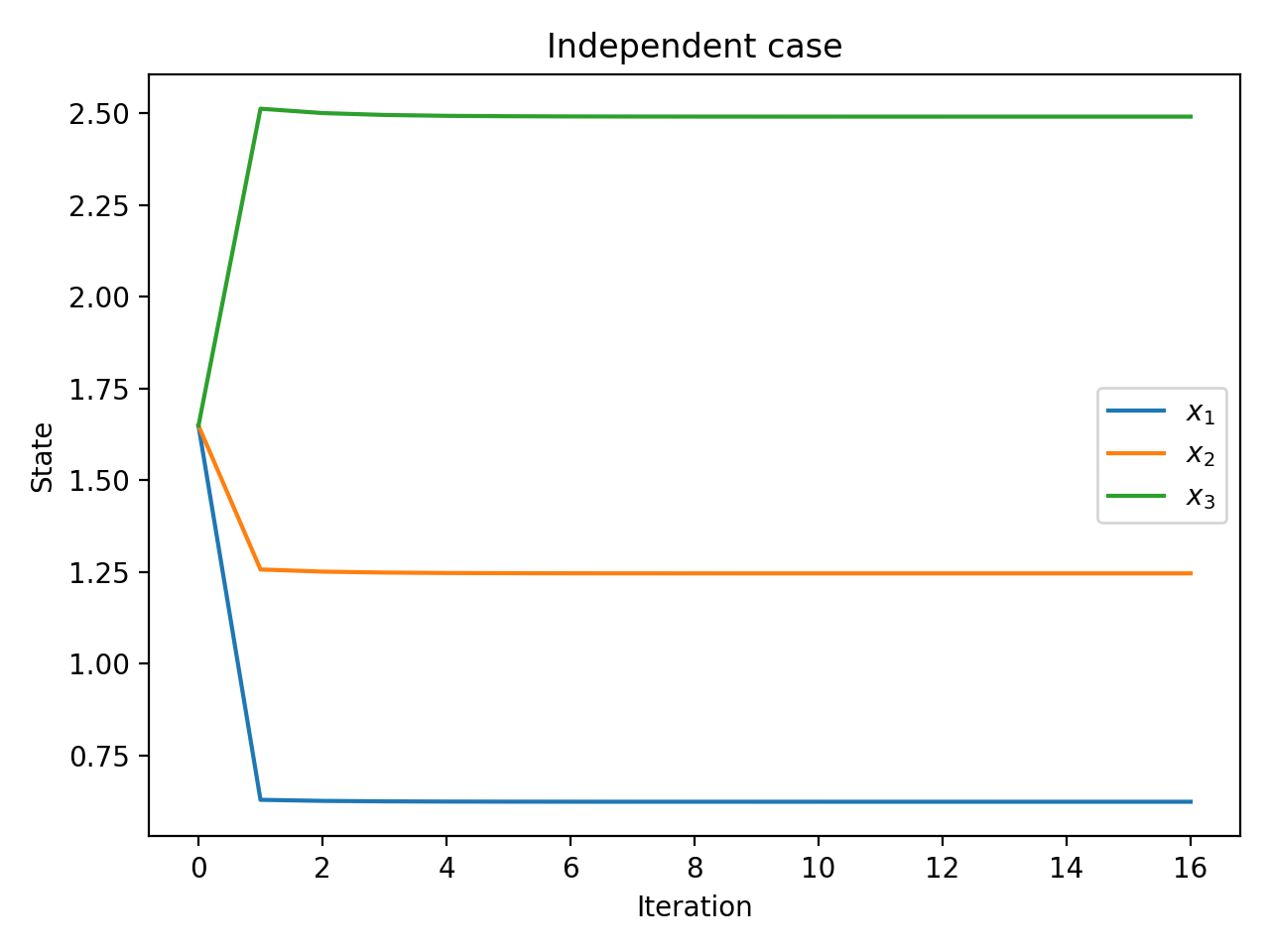}
\end{subfigure}%
\hspace{0.7 cm}\begin{subfigure}{.5\textwidth}
  \centering
 \includegraphics[scale = 0.4]{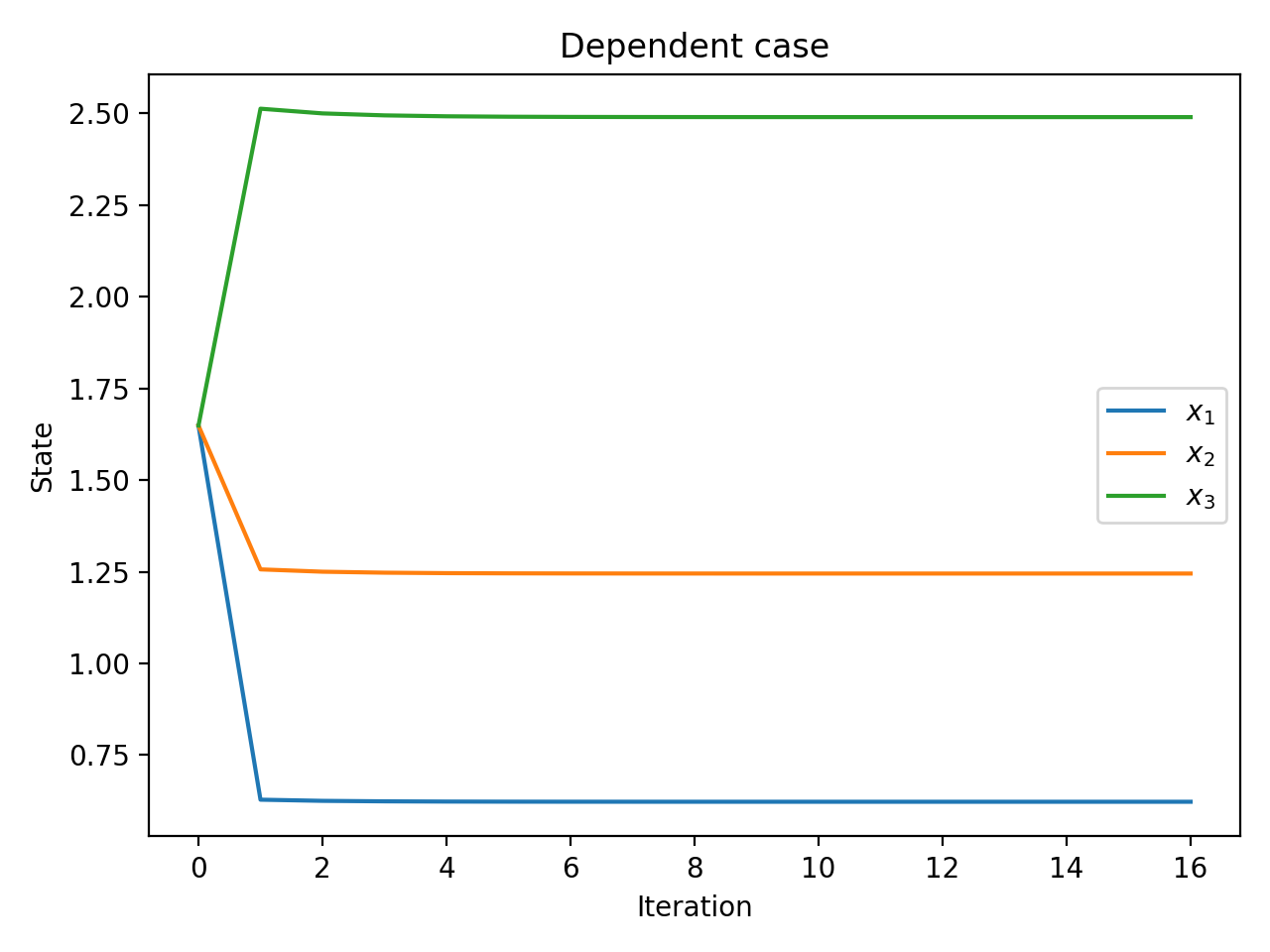}
\end{subfigure}%
\caption{{Transient behaviors of the state variables}}
\label{fig:convergence_shape_optimization_problem}
\end{figure}

\subsubsection{Multidimensional shape optimization problem}
\noindent To further assess the performance of our dynamical neural network, we use the multidimensional shape optimization problem with joint chance constraints from \cite{liu2022distributionally}.

\begin{eqnarray}
&\min_{x \in \mathbb{R}_{++}^M} & \prod_{i=1}^m x_i^{-1}, \notag\\
&\text{s.t}& 
          \inf\limits_{\mathcal{F} \in \mathcal{D} }\mathbb{P}_{\mathcal{F}}\left(  \sum\limits_{j=1}^{m-1}(\frac{m-1}{{A_{wall}}_j}x_1 \prod_{i=1, i\neq j}^m x_i)\leq 1,  \frac{1}{A_{floor} }\prod_{j=2}^m x_j  \leq 1 \right)  \geq 1-\epsilon, \label{eq:multidimensional_shape_optimization} \\ 
& &         \frac{1}{\gamma_{i,j} }x_ix_j^{-1} \leq 1, 1\leq i\neq j \leq m. \notag
\end{eqnarray}
In our numerical experiments, we fixed the following parameters
$\frac{1}{\gamma_{i,j}}=0.5$ and $\epsilon = 0.15$. The inverse of floor's area
($\frac{1}{A_{floor}}$) and the inverse of wall area ($\frac{1}{A_{wall_j}}$)
for each $j=1, ..., m$ were considered as random variables.  {To define the
uncertainty set $\mathcal{D}^2$, we generate the components of $\mu$ uniformly
in $[1.0/40.0, 1.0/20.0]$ and the components of the matrix $\Sigma$  uniformly
in $[0.001, 0.01]$ } We test the robustness of the different approaches by
creating 100 random samples of the variables $\frac{1}{A_{wall_j}}$ and
$\frac{1}{A_{floor}}$. We then examine if the solutions meet the constraints of
(\ref{eq:multidimensional_shape_optimization}) for all 100 cases for the normal
distribution. If the solutions are not feasible for a particular case, it is
referred to as a violated scenario (VS).\\
\\
We first solve (\ref{eq:multidimensional_shape_optimization}) for $m=5$ and when
the uncertainty set is $\mathcal{D}^2$ in the independent case for different
initial points, we observe that the dynamical neural network
(\ref{eq:KKT_1})-(\ref{eq:KKT_2}) converges to the same final value
independently from the starting value as shown in Figure
\ref{fig:convergence_joint}.

\begin{figure}[!htb] 
\hspace{1.0 cm}{
\begin{center}
 \includegraphics[scale=0.5]{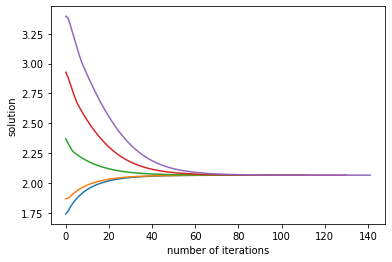}
 \caption{Convergence of the dynamical neural network (\ref{eq:KKT_1})-(\ref{eq:KKT_2}) for different initial points for (\ref{eq:multidimensional_shape_optimization}).}
 \label{fig:convergence_joint}
 \end{center}}
\end{figure}

\noindent Now we solve (\ref{eq:multidimensional_shape_optimization}) for known first-order moments of $\frac{1}{A_{floor}}$ and $\frac{1}{A_{wall_j}}$ for both the dependent and the independent case. We present the obtained results in Table \ref{tab:results_different_m_values}. We observe again that the dependent case is more conservative compared to the independent one since the final values of the objective functions are lower. Nevertheless, the dependent case covers the risk area better since the number of violated scenarios is lower compared to the independent case.

\begin{scriptsize}
\begin{table}[t]
    \hspace{1.0 cm} \begin{tabular}{lllllllllllll}
  \hline
  $m$ & & \multicolumn{3}{c}{Independent case} &&  \multicolumn{3}{c}{ Dependent case}\\
 \cline{3-5}\cline{7-9}
  & & Obj value & CPU Time & VS &  & Obj value & CPU Time &VS  \\ 
   \cline{3-5}\cline{7-9}
 3 & &1.03 &1.05 &4 & &  1.30 & 1.39 & 0 &\\
   \hline
 5 & &2.09 &5.11  &3& & 2.15 & 5.20 & 1 & \\
   \hline
  10 && 14.79  & 4.83  & 2 && 15.10 & 5.04 &0& \\
   \hline
   15  & & 7.76 & 47.80 & 3 & & 7.99 & 58.04 & 1 &\\
   \hline
   20 & & 10.68 & 97.72 & 2 && 10.87 & 100.91 & 0 & \\
   \hline
\end{tabular}
    \caption{Results for different values of $m$}
    \label{tab:results_different_m_values}
\end{table}
\end{scriptsize}

\subsection{Uncertainty Sets with Known First Order Moment and Nonnegative Support}
\label{subsec:numerical_exp_2}

\subsubsection{A generalized shape optimization problem}

\noindent We solve (\ref{eq:multidimensional_shape_optimization}) when the
uncertainty set is $\mathcal{D}^3$ for both the independent and the dependent
case.\\
\\
For the numerical experiments, we take $\epsilon=0.2$. We solve problem (\ref{eq:multidimensional_shape_optimization}) using the neurodynamic duplex in the dependent case. We recapitulate the obtained results in Table \ref{tab:results_different_m_values_D_3}. Column 1 gives the number of variables $m$. Columns 2, 3 and 4 give the objective value, the CPU time and the number of VS, with respect to the normal distribution, in the independent case, respectively. Columns 5, 6, and 7 provide the objective value, CPU time, and the number of VS for the dependent case, respectively. We observe that the problem with dependent variables is more conservative. Nevertheless, the solution in this case covers the risk area well, as the number of VS is equal to 0 for all values of $m$. 
\begin{scriptsize}
\begin{table}[t]
    \hspace{1.0 cm} \begin{tabular}{lllllllllllll}
  \hline
  $m$& & \multicolumn{3}{c}{Independent case} &&  \multicolumn{3}{c}{ Dependent case}\\
 \cline{3-5}\cline{7-9}
  & & Obj value & CPU Time & VS &  & Obj value & CPU Time &VS  \\ 
   \cline{3-5}\cline{7-9}
   3& &0.204&2.28 &3  & & 0.491& 10.12& 0 \\
   \hline
  5&&1.03&6.25 & 2&&1.82& 98.68&  0\\
   \hline
   10 &&6.99 &15.26& 2 &&9.79 &86.35 &0  \\
   \hline
   15 &&18.43 &23.84 &3  &&23.45  &201.13& 0  \\
   \hline
   20 &&32.09&94.76& 5  &&38.71 &744.26&0   \\
   \hline
    30 &&42.37&100.23& 3  &&51.56&1155.42 &0   \\
   \hline
\end{tabular}
    \caption{(\ref{eq:multidimensional_shape_optimization}) for different values of $m$ for $\mathcal{D}= \mathcal{D}^3$}
    \label{tab:results_different_m_values_D_3}
\end{table}
\end{scriptsize}
\\
Now we additionally solve problem (\ref{eq:multidimensional_shape_optimization}) using a stochastic approach with the assumption that the random variables follow a normal distribution \cite{TASSOULI2023765} for $m=5$. In order to compare the solutions obtained with the stochastic and the robust approaches, we evaluate the robustness of the solutions for different hypotheses on the true distribution of the random parameters, i.e., the uniform distribution, the normal distribution, the log-normal distribution, the logistic distribution and Gamma distribution. The obtained results are presented in Table \ref{tab:violated_scenarios} which gives the number of violated scenarios for both the distributed normal solutions and the robust ones and the objective value obtained by each solution. We can infer that the distributionally robust approaches are a
conservative approximation of the stochastic programs. We observe that the solutions obtained by the nonnegative support are more conservative compared to the stochastic ones. Notice that the distributionally robust solutions are more robust, i.e., the number of VS when the true distribution is the Logistic distribution is equal to 23 and 19 for the nonnegative support solutions and is equal to 0 for the robust solutions.

\begin{table}[htbp]
    \hspace{0.5 cm} \begin{tabular}{l|llllllll}
  \hline
  & & &\multicolumn{2}{c}{Normal solutions} & & \multicolumn{2}{c}{ Robust solutions }   \\
    \cline{4-5} \cline{7-8}
  & &&Independent & Dependent   &&Independent & Dependent \\ \hline
& Objective Value && 0.86 & 0.99 && 2.43 & 4.14 &\\
\cline{2-9}
  Number & Uniform distribution& & 22 & 15 && 0 & 0\\
   
of  & Normal distribution& & 18 & 11  && 1  &0 \\
   
 violated & Log-normal distribution && 7 & 4 && 2 & 1 \\
  
 scenarios & Logistic distribution& & 23 & 19 && 0 &0 \\
 
  &  Gamma distribution & & 16 & 12 && 2 & 2 \\
   \hline
\end{tabular}
    \caption{Number of violated scenarios for the stochastic and the robust solutions}
    \label{tab:violated_scenarios}
\end{table}

\subsubsection{Maximizing
the worst user signal-to-interference noise ratio}
\noindent We consider the problem of maximizing
the worst user signal-to-interference noise ratio (SINR) for
Massive Multiple Input Multiple Output (MaMIMO) systems subject to antenna assignment and multiuser interference constraints taken from \cite{Adasme_Lisser_2023} and given by 
\begin{eqnarray}
&\max\limits_{p \in {\nbR_{++}^K}} \min\limits_{i \in \mathcal{U}} &  \frac{p_i|g_i^Hg_i|^2}{\sum\limits_{j \in \mathcal{U}, j \neq i}p_j|g_i^Hg_j|^2 + |\sigma_i|^2}, \label{eq:MaMIMO_max_min_obj}  \\
&\text{s.t}& P_{min} \leq p_i \leq P_{max}, \forall i \in \mathcal{U}, \label{eq:MaMIMO_max_min_constraint}
\end{eqnarray}
where $p_i$ is the power to be assigned for each user $i \in \mathcal{U}$. $g_i \in \mathbb{C}^{T \times 1}$, $g_i^H \in \mathbb{C}^{1 \times T}$ and $\sigma_i^2$  are the
beam domain channel vector associated to user $i \in \mathcal{U}$,
its Hermitian transpose and Additive White Gaussian Noise
(AWGN), respectively.\\
\\
Let $a_{ij} = |g_i^Hg_j|^2|g_i^Hg_i|^{-2} $ and $b_i =  |\sigma_i|^2 |g_i^Hg_i|^{-2} $, we derive a geometric reformulation of (\ref{eq:MaMIMO_max_min_obj})-(\ref{eq:MaMIMO_max_min_constraint})

\begin{eqnarray} 
&\min\limits_{p \in {\nbR_{++}^K}, w \in \nbR_{++} } &w^{-1} \notag\\
&\text{s.t}& 
      \sum\limits_{j \in \mathcal{U}, j \neq i} a_{ij}p_jp_i^{-1}w + b_ip_i^{-1}w \leq 1 \quad \forall i \in \mathcal{U} \notag\\
& &   P_{min} \leq p_i \leq P_{max} \quad \forall i \in \mathcal{U} \notag
\end{eqnarray}
We assume that the coefficients $a_{ij}$ and $b_i$ are independent random variables and we propose the following optimization problem with individual and joint chance constraints

\begin{eqnarray} 
&\min\limits_{p \in {\nbR_{++}^K}, w \in \nbR_{++} } &w^{-1},\notag  \\
&\text{s.t}& 
     \inf\limits_{\mathcal{F}_i \in \mathcal{D}_i } \mathbb{P}_{\mathcal{F}_i } \left\{  \sum\limits_{j \in \mathcal{U}, j \neq i} a_{ij}p_jp_i^{-1}w + b_ip_i^{-1}w \leq 1 \right\} \geq 1-\epsilon_i,  \forall i \in \mathcal{U} \label{POI}\tag{$POI$}\\
& &       P_{min} \leq p_i \leq P_{max} \quad \forall i \in \mathcal{U}.\notag
\end{eqnarray}
and 

\begin{eqnarray} 
&\min\limits_{p \in {\nbR_{++}^K}, w \in \nbR_{++} } &w^{-1}, \notag  \\
&\text{s.t}& 
     \inf\limits_{\mathcal{F} \in \mathcal{D} } \mathbb{P}_{\mathcal{F}}  \left\{  \sum\limits_{j \in \mathcal{U}, j \neq i} a_{ij}p_jp_i^{-1}w + b_ip_i^{-1}w \leq 1, \forall i \in \mathcal{U}  \right\} \geq 1-\epsilon \label{POJ}\tag{$POJ$}\\
& &  P_{min} \leq p_i \leq P_{max} \quad \forall i \in \mathcal{U}\notag
\end{eqnarray}
We assume that the uncertainty set for the distributionally robust problems
(\ref{POI}) and (\ref{POJ}) is $\mathcal{D}^3$. We fix $\epsilon=0.2$. We first
solve problem (\ref{POJ}) for $K=10$. Figure \ref{fig:powersjoint} shows the
convergence of the power variables. Next, we solve (\ref{POI}) and (\ref{POJ})
for different values of the number of users $K$. Table
\ref{tab:results_different_K_values} presents the obtained results. Column 1
gives the number of users $K$. Columns 2 and 3 give the optimal value and the
number of VS, with respect to the normal distribution, for (\ref{POI}),
respectively. Columns four and five show the optimal value and the number of VS
for (\ref{POJ}), respectively. As observed in the previous Section
\ref{sec:two-time_scale_neurodynamic_duplex}, the use of joint constraints leads
to a more conservative minimization problem but covers well the risk area
compared to the problem with individual constraints since the number of VS is
lower. 
\begin{figure}
\hspace{1.0 cm}{
\begin{center}
 \includegraphics[scale=0.5]{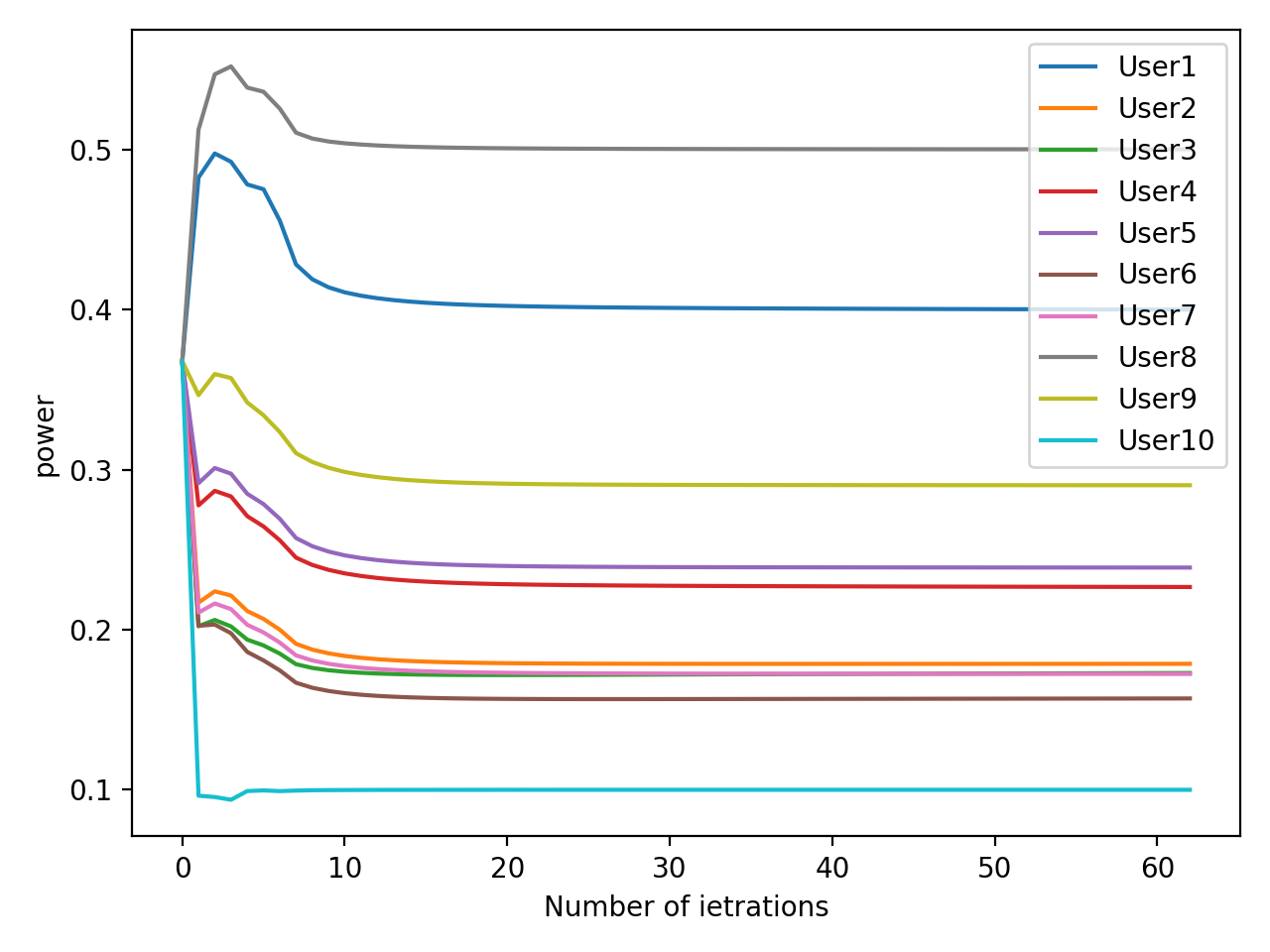}
 \caption{Convergence of the power variables}
  \label{fig:powersjoint}
 \end{center}}
\end{figure}\\

\begin{scriptsize}
\begin{table}[t]
    \hspace{4.0 cm} \begin{tabular}{lllllllllll}
  \hline
  $K$ & & \multicolumn{2}{c}{Individual constraints} && \multicolumn{2}{c}{ Joint constraints}\\
 \cline{3-4}\cline{6-7}
& & Obj value & VS & & Obj value &VS  \\ 
   \cline{3-4}\cline{6-7}
5& &27.27&5 & &29.07& 0  \\
   \hline
10 &&47.36&4 &&50.23&0  \\
   \hline
15&&66.03& 5&&68.76 &1 \\
   \hline
20 &&123.48&3 &&127.43 &0  \\
   \hline
\end{tabular}
    \caption{Results for different values of $K$}
    \label{tab:results_different_K_values}
\end{table}
\end{scriptsize}

\noindent Now, we compare our neurodynamic duplex to the CAR method. We recapitulate the obtained results in Table \ref{tab:neurodynamic_duplex_alternate_convex_search}.

\begin{scriptsize}
\begin{table}[t]
    \hspace{1.0 cm} \begin{tabular}{llllllllllll}
  \hline
  $K$ & & \multicolumn{2}{c}{The neurodynamic Duplex } &&  \multicolumn{2}{c}{ The Alternate Convex Search} && GAP\\
 \cline{3-4}\cline{6-7}
& & Obj value & VS & & Obj value &VS && \\ 
   \cline{3-4}\cline{6-7}
5& &29.03&1 & &31.14& 4&& 6.77\%  \\
   \hline
10 &&50.23&1 &&52.04&3 && 3.47\% \\
   \hline
15&&66.03& 0&&68.66 &4&& 3.82\% \\
   \hline
20 &&123.48&1 &&125.34 &5  &&1.48\%\\
   \hline
30 &&187.55&1 &&189.36 &4  && 0.95\%\\
   \hline
40 &&213.89&2 &&210.43 &4  &&1.61\%\\
   \hline
\end{tabular}
    \caption{The neurodynamic Duplex vs. The Alternate Convex Search }
    \label{tab:neurodynamic_duplex_alternate_convex_search}
\end{table}
\end{scriptsize}

\noindent For the different values of $K$, the GAP between both methods does not
exceed $6\%$. Therefore, they yield similar performance results, as the
differences in constraint violations and objective values are not significant.
The advantage of considering the neurodynamic duplex approach over the CAR
method is that it allows for obtaining solutions to multiple instances of the
same problem without requiring further training. We can solve different
instances of the problem by changing the parameter $\theta$, whereas the CAR
method must be applied to each instance. Once the training of the conditional
recurrent neural networks $\text{RNN}_1$ and $\text{RNN}_2$ in the neurodynamic
duplex has been done, the model can directly output a solution adapted to a
change in the numerical values of $(a_{i,j})_{i,j \in \mathcal{U}, i \neq j}$,
$(b_i)_{i \in \mathcal{U}}$ or $\epsilon$ in (\ref{POJ}), represented by
$\theta$, which is a flexibility that the CAR algorithm could not offer. In
Table \ref{tab:multiple_instances}, we show the computation time for multiple
instances of the problem solved by both methods. We observe that our
neurodynamic duplex is 100 times faster than the alternate convex search when
solving 100 instances of the problem, since solving multiple instances takes the
same CPU time as solving one, whereas the alternate convex search method must be
applied again entirely for each instance.

\begin{scriptsize}
\begin{table}[t]
    \hspace{1.0 cm}
\begin{tabular}{llll}
  \hline
  Number of instances & The neurodynamic Duplex && The Alternate Convex Search\\
 \cline{2-2}\cline{4-4}
1 & 5.11 && 5.03 \\
   \hline
10 & 4.85 && 50.06 \\
   \hline
20 & 5.22 && 100.21 \\
   \hline
50 & 5.06 && 250.40 \\
   \hline
100 & 5.01 && 500.62
\end{tabular}
\caption{Comparison of CPU times, in seconds, between the neurodynamic
Duplex and the Alternate Convex Search on solving multiple instances of
(\ref{POJ}).}
\label{tab:multiple_instances}
\end{table}
\end{scriptsize}


\section{Conclusion}

\noindent This paper studies a distributionally robust joint-constrained
geometric optimization problem for two different moments-based uncertainty sets.
We propose two neurodynamic approaches to solve the resulting optimization
problems. To assess the performances of the proposed approaches, we solve a
problem of shape optimization and a telecommunication problem. We observe that
our approach is robust since it covers well against the fluctuations of the risk
area. The neurodynamic approach exhibits a key advantage in its ability to
converge closer to the optimal solution, whereas the alternate convex search
method only provides an upper bound due to its convex approximation. In our
numerical experiments, we obtain similar results for both methods in terms of
objective value and constraints violations, but the neurodynamic duplex is
significantly faster than the alternate convex search in CPU times for solving
multiple instances of a problem. This distinction underscores the superior
accuracy and effectiveness of our approach.\\
\\
We note that the performances of our approaches can be significantly increased
with the development of new ODE solvers mainly based on machine learning
techniques.

\bibliography{mybibfile}

\end{document}